\documentclass{article}

\usepackage{microtype}
\usepackage{graphicx}
\usepackage{subfigure}
\usepackage{balance}
\usepackage{booktabs} 

\usepackage{hyperref}

\usepackage[accepted]{_sty/icml2024}

\usepackage{amsmath}
\usepackage{amssymb}
\usepackage{mathtools}
\usepackage{amsthm}
\usepackage{comment}

\usepackage{hyperref}       
\usepackage{url}            
\usepackage{booktabs}       
\usepackage{amsfonts}       
\usepackage{nicefrac}       
\usepackage{microtype}      
\usepackage{xcolor}         
\usepackage{comment}
\usepackage{graphicx} 
\usepackage{multirow}
\usepackage{wrapfig}

\usepackage[capitalize,noabbrev]{cleveref}

\theoremstyle{plain}
\newtheorem{theorem}{Theorem}[section]
\newtheorem{proposition}[theorem]{Proposition}
\newtheorem{lemma}[theorem]{Lemma}
\newtheorem{corollary}[theorem]{Corollary}
\theoremstyle{definition}
\newtheorem{definition}[theorem]{Definition}
\newtheorem{assumption}[theorem]{Assumption}
\theoremstyle{remark}
\newtheorem{remark}[theorem]{Remark}

\usepackage[textsize=tiny]{todonotes}

\icmltitlerunning{Out-of-Distribution Detection via Deep Multi-Comprehension Ensemble}

\begin{document}

\twocolumn[
\icmltitle{Out-of-Distribution Detection via Deep Multi-Comprehension Ensemble}



\icmlsetsymbol{equal}{*}

\begin{icmlauthorlist}
\icmlauthor{Chenhui Xu}{gmu}
\icmlauthor{Fuxun Yu}{gmu}
\icmlauthor{Zirui Xu}{gmu}
\icmlauthor{Nathan Inkawhich}{afrl}
\icmlauthor{Xiang Chen}{gmu}
\end{icmlauthorlist}

\icmlaffiliation{gmu}{George Mason University}
\icmlaffiliation{afrl}{Air Force Research Laboratory}

\icmlcorrespondingauthor{Xiang Chen}{xchen26@gmu.edu}
\icmlcorrespondingauthor{First Author: Chenhui Xu}{cxu21@gmu.edu}

\icmlkeywords{Machine Learning, ICML}

\vskip 0.3in
]



\printAffiliationsAndNotice{}  

\begin{abstract}

    Recent research works demonstrate that one of the significant factors for the model Out-of-Distribution detection performance is the scale of the OOD feature representation field.   
        Consequently, model ensemble emerges as a trending method to expand this feature representation field leveraging expected model diversity.
However, by proposing novel qualitative and quantitative model ensemble evaluation methods~(i.e., Loss Basin/Barrier Visualization and Self-Coupling Index), we reveal that the previous ensemble methods incorporate affine-transformable weights with limited variability and fail to provide desired feature representation diversity.
    Therefore, we escalate the traditional model ensemble dimensions (different weight initialization, data holdout, etc.) into distinct supervision tasks, which we name as Multi-Comprehension (MC) Ensemble. 
        MC Ensemble leverages various training tasks to form different comprehensions of the data and labels, resulting in the extension of the feature representation field.
In experiments, we demonstrate the superior performance of the MC Ensemble strategy in the OOD detection task compared to both the naive Deep Ensemble method and the standalone model of comparable size.
\end{abstract}
\section{Introduction}
State-of-the-art neural network models often exhibit overconfidence in their predictions due to their training and generalization in a static and closed environment. 
    Specifically, these models assume that the distribution of test samples is identical to that of the training samples. 
        However, this assumption may not hold in the open world, as out-of-distribution (OOD) samples can arise from unreliable data sources or adversarial attacks. 
            Such OOD samples can introduce significant challenges to the generalization performance of these models.
Due to reliability and safety concerns, it is crucial to identify when input data is OOD. 
    Significant research efforts have been devoted to detecting OOD samples~\cite{liang2018enhancing,hendrycks2019using,ren2019likelihood,huang2021importance}, as well as the estimation of uncertainty~\cite{lakshminarayanan2017simple} in neural network models.

In practice, researchers proposed combining multiple independent models to enhance the robustness of model predictions against OOD samples~\cite{lakshminarayanan2017simple,zaidi2021neural,Malinin2020Ensemble,kariyappa2021protecting,li2021kfolden,xue2022boosting}. 
Inspired by the Bagging~\cite{breiman1996bagging}, one of the most representative works --- Deep Ensembles~\cite{lakshminarayanan2017simple} was proposed, which calculates the average of the posterior probabilities generated by multiple models with different initializations. 
    This approach delivers an ensemble model that is more pervasive and scalable for OOD detection. 


\begin{table*}
  \caption{Ensemble }
  \centering    
  \tabcolsep 6pt
  \resizebox{\textwidth}{!}{
  \begin{tabular}{lccc}
    \toprule
\textit{Method} & \textit{Diversity Approach} & \textit{Training Criterion} & \textit{Comprehension}\\
    \midrule
Deep Ensemble~\cite{lakshminarayanan2017simple} & Weight Initialization & Cross-Entropy~(CE) & Data-Label Pairs\\
SSLC~\cite{vyas2018out} & Data Leave-Out & Margin-Entropy & Data-Distribution Pairs\\
kFolden~\cite{li2021kfolden} & Data Leave-out & Cross-Entropy & Data-Label Pairs\\
EnD$^2$~\cite{Malinin2020Ensemble} & Weight Initialization & Cross-Entropy & Data-Label Pairs\\
LaCL~\cite{cho2022enhancing} & Weight Initialization & Supervised Contrastive & Bring similar data close\\
\textbf{MC Ensemble}~(ours)& Training Task & CE+SimCLR+SupCon &\textbf{Multi-Comprehension} \\

    \bottomrule
  \end{tabular}
  }
  \label{tab:comprehension}
\end{table*}
            
However, recent work~\cite{abe2022deep} claims that the ensemble diversity does not meaningfully contribute to a Deep Ensemble's OOD detection performance improvement. 
    This means that Deep Ensemble's performance is consistent with that of a single model of equivalent size.
        We observe this phenomenon and attribute it to the fact that the diversity provided by naive Deep Ensemble through different model initializations is not significant enough.
     Specifically, the individuals in a naive ensemble tend to exhibit a considerable lack of diversity in feature representation.
        
This is because, although the individual models of the deep ensemble seem different due to diverse initializations and partial datasets, they still adopt the same training criterion, forming a monotonic comprehension.
    As shown in Table~\ref{tab:comprehension}, for example, models trained with cross-entropy loss always try to directly find the mapping from data to labels. 
        The formation of this single comprehension is accompanied by intrinsic mode connectivity~\cite{pagliardini2022agree,frankle2020linear,ainsworth2023git} among neural networks. 
            With only a single training criterion, individual models in an ensemble are usually mode-connected and thus are not sufficient to generate a diversity of feature representations that can boost the OOD detector.
            
As diversity is the key to the model ensemble, we propose a new perspective to measure it regarding the distribution distance between feature representations. Notably, we illustrate how different training tasks can give diverse feature representations to the models in terms of the loss landscape.
    Based on feature representation and loss landscape perspective findings and assumptions, we demonstrate that the training task is a crucial factor in the diversity of the models. 
Therefore, we devised a novel ensemble scheme, named Multi-Comprehension Ensemble~(MC Ensemble) that integrates models trained on different tasks but with the same structure and training data. 
    Our ensemble breaks away from the original ensemble approach in the dimensionality of single comprehension patterns. 
        We bring a new dimension to the consideration of ensemble diversity: the\textbf{ comprehension mode} of models.
    Our experiments show that this ensemble scheme outperforms other ensemble approaches like different initialization and data leave-out on CIFAR10 and Imagenet Benchmarks.

\textbf{Contributions.} We make the following contributions:
\begin{itemize}
\vspace{-3mm}
    \item We demonstrate the feasibility of feature-level ensemble in OOD detection in principle. (Section \ref{S2})
    
\vspace{-2mm}
    \item We reveal that the previous ensembles' inability to effectively detect OOD samples can be attributed to the insufficient level of diversity among models trained using the same criterion. (Section \ref{S3})

    \vspace{-2mm}
    \item We propose a novel method, Self-Coupling Index, to quantitatively measure the difference between feature representations generated by two models. (Section \ref{S3})
        \vspace{-2mm}
    \item We reveal that multiple training criteria introduced by different supervision tasks can make the loss barrier between models larger through the perspective of the loss landscape, thus enabling diverse penultimate-layer feature representations, and eventually, forming a diverse Multi-Comprehension mode.  (Section \ref{S4})
        \vspace{-2mm}
    \item We propose a feature-level ensemble scheme that exploits the diversity of models based on distinct comprehension, resulting a model powerful in OOD detection.
\end{itemize}


\section{How Feature Ensembles Boost OOD Detection?}
\label{S2}

Let $\mathcal{X}$ and $\mathcal{Y}$ be the input space and label space. 
    We define the penultimate layer representation space of the neural network as $\mathcal{R}$. 
Then the neural network trained on hypothesis $H$ can be represented as $f_H(x) = h_H(g_H(x)), x\in\mathcal{X} $, where $g_H: \mathcal{X} \to \mathcal{R}$ is the feature encoder and $h_H: \mathcal{R} \to \mathcal{Y}$ is the projection head. 
    The hypothesis $H$ contains the training criterion (i.e. Cross-entropy loss, SimCLR~\cite{chen2020simple}, SupCon loss~\cite{khosla2020supervised}), data distribution $D$, initialization $\Theta$, and other training configuration.

\begin{figure*}[t]
    \centering
    \includegraphics{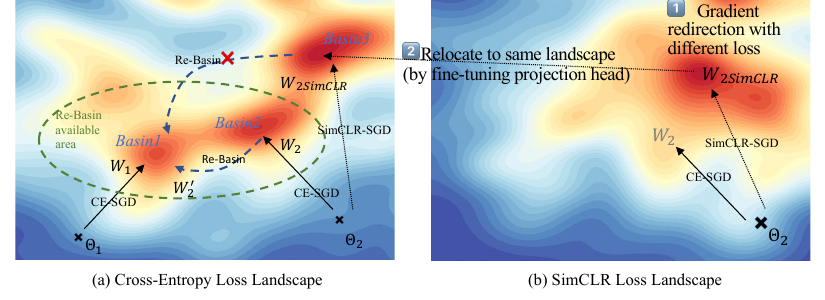}
    \caption{(a)~Models trained by \textbf{different initialization} ($\Theta_1$, $\Theta_2$) but with the same cross-entropy classification task (\textit{CE-SGD}) can fall into the same or symmetric loss basin, which can be affine-transformed into the same basin (\textit{Re-Basin}~\cite{ainsworth2023git}). This indicates the two models provide little variability. (b)~By contrast, a \textbf{different comprehension task} (\textit{SimCLR-SGD}) directs the model parameters in other directions. When \textit{SimCLR-SGD} weights are relocated to the same loss landscape of \textit{CE-SGD} weights, we can observe the loss barrier between two sets of weights is high so that \textit{Re-Basin} is not possible, thus increasing the model and feature diversity. }
    \label{fig:rebasin}
\end{figure*}  
\subsection{What Is a Good OOD Detection Booster?}


    A simple and common OOD detection method is to score the feature representation $z = g_H(x) = (z_1,z_2,\dots,z_m) \in \mathcal{R}$ in the penultimate layer space based on the scoring metrics $s(z) \in \mathbb{R}$~\cite{liu2020energy,sun2022out,liang2018enhancing}. 
        Then we determine the sample by its score and a threshold $\tau$ that the sample is ID if $s(z)>\tau$ and vice versa OOD. 
    Previous work has proved that the separation of ID and OOD in feature space can be transferred into OOD detector's output space
    \cite{sun2021react}, indicating that a good OOD detection booster should make ID and OOD activation more separable. 
        Considering the nature difference the distribution between ID~(following a rectified normal distribution, $z_i \sim \mathcal{N}^R(\mu,\sigma^2)$) and OOD activation~(following a rectified epslion-skew normal distribution,$z_i \sim ESN^R(\mu,\sigma^2,\epsilon)$)~\cite{sun2021react}, to make two distributions more separable, we should have: 
        
\textbf{(1) ID data should achieve greater positive activation movement (increase) compared to OOD data in average.} We denote $\bar{z_i}$ as the activation after applying an OOD detection booster. Then:
\begin{equation}
    \mathbb{E}_{out}[\Bar{z_i}-z_i]-\mathbb{E}_{in}[\Bar{z_i}-z_i] \leq 0
    \label{eact}
\end{equation}
\textbf{(2) Activation after boosting should form a better estimate of the parameter $\mu$.} In other words, the variance of the estimate should be smaller than the not-boosted one,
\begin{equation}
    Var(\bar{\hat{\mu}}) \leq  Var({\hat{\mu}}).
    \label{evar}    \vspace{-1mm}
\end{equation}

\subsection{Feature-level Ensemble Is a Good OOD Booster}
The traditional ensemble strategy is based on the bias-variance decomposition theory~(see Appendix \ref{A1}). However, this theory ignores the ensemble's effect on feature representation, and thus in principle fails to explain ensemble-based OOD detection in feature level. We first demonstrate that individuals in feature-level ensemble will not counteract each other, in Appendix~\ref{A2.1}. Then, we analyze the feature-level ensemble from the above two conditions.

Under the premise of using neural networks with the same architecture, we assert that the pre-activation features of a single dimension in different models follow the same distribution due to the presence of normalization layer. For ID data, compared with a single model, the average movement of feature averaging ensemble for activation is:
\begin{multline}
     \mathbb{E}_{\text{in}}[\bar{z_i}-z_i] = \mu\left[\Phi\left(\frac{\mu\sqrt{M}}{\sigma_{\text{in}}}\right)-\Phi\left(\frac{-\mu}{\sigma_{\text{in}}}\right)\right]+\\\sigma_{\text{in}}\left[\frac{1}{\sqrt{M}}\phi\left(\frac{\mu\sqrt{M}}{\sigma_{\text{in}}}\right)-\phi\left(\frac{-\mu}{\sigma_{\text{in}}}\right)\right].
\end{multline}

For OOD data, the corresponding movement will be:
\begin{align}
    &\mathbb{E}_{\text {out}}\left[\bar{z_i}-z_i\right] = \frac{4\epsilon\sigma_{\text{out}}}{\sqrt{2\pi}}\left(1-\frac{1}{\sqrt{M}}\right) \\ &+(1+\epsilon)\mu\left[\Phi\left(\frac{\mu\sqrt{M}}{(1+\epsilon)\sigma_{\text{out}}}\right)-\Phi\left(\frac{\mu}{(1+\epsilon)\sigma_{\text{out}}}\right)\right]+\notag\\&(1+\epsilon)^2\sigma_{\text{out}}\left[\frac{1}{\sqrt{M}}\phi\left(\frac{\mu\sqrt{M}}{(1+\epsilon)\sigma_{\text{out}}}\right)-\phi\left(\frac{\mu}{(1+\epsilon)\sigma_{\text{out}}}\right)\right]. \notag
\end{align}
Under the same chaotic-level~($\sigma_{\text{in}}=\sigma_{\text{out}}$), the activation movements satisfy Eq.~(\ref{eact}), that is ID activations move more. See Appendix \ref{A2} for detailed proof.

Meanwhile, the essence of the process of selecting a certain number of models from a model pool to construct a feature average ensemble is sampling. The construction process is done according to some rules (we artificially design the content models that makes up ensemble), which is consistent with the characteristics of cluster sampling with small cluster~\cite{angrist2009mostly}. Therefore, the variance estimate of the parameter $\mu$ with averaged feature will be:
\begin{equation}
    Var(\bar{\hat{\mu}}) = \frac{1+(M-1)\rho}{M}Var({\hat{\mu}}) \leq Var({\hat{\mu}}),
    \label{eqvar}
\end{equation}
where $\rho \in [0,1]$ denotes intraclass correlation coefficient. 
    
Therefore, satisfying Eq.~(\ref{eact}) and
(\ref{evar}), we conclude feature-level ensemble is a good OOD detection booster.

\subsection{Feature Diversity Matters in Ensemble}
 Intraclass correlation directly reflects the diversity of individual models in ensemble.
    Due to the same training data, network architecture, or training criteria, feature representations in ensemble fails to be independent, leading to a non-zero intraclass correlation. 
 Eq.~(\ref{eqvar}) reveals that with a smaller intraclass correlation, the ensemble will be stronger to separate ID and OOD.
        However, direct measurement of intraclass correlation fails to reveal the model difference, because the intraclass correlation is restricted to low-dimensional statistics while the dimension to compare two models' representation at least in the order of millions (\# of samples $\times$ feature dimensions). 
            This requires us to reconsider how we measure the diversity of models.

\section{How Much Diversity Exists among Models?}
\label{S3}
\subsection{Diversity: Mode Connectivity and Feature View}

\begin{figure*}[t]
    \centering
    \includegraphics{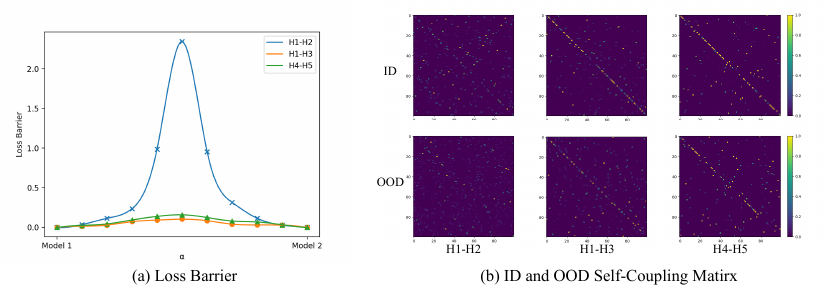}
    \caption{Two models trained from different hypotheses. When there is a large loss barrier between models, the coupling matrix of features tends to perform more stochastic. The models' architecture is ResNet-18~\cite{he2016deep}. H1: Initialization $\Theta_1$, Cross-Entropy loss, whole training set. H2: Initialization $\Theta_2$, SimCLR Loss, whole training set. H3: Initialization $\Theta_2$, Cross-Entropy Loss, whole training set. H4: Initialization $\Theta_1$, Cross-Entropy Loss, 80\% training set. H5: Initialization $\Theta_2$, Cross-Entropy Loss, another 80\% training set.}
    \label{fig:barrier}
\end{figure*}

\label{pare:3}
Since the post-hoc scoring metric in OOD detection employs penultimate layer feature maps, the penultimate layer feature diversity will be important in ensemble-based approaches. 
Deep Ensembles~\cite{lee2018simple,fort2019deep} claim the diversity via randomness of SGD coupled with non-convex loss surfaces. 
    While other data-based ensembles (i.e. K-fold~\cite{li2021kfolden} and Bagging~\cite{breiman1996bagging}) use the differences in training data from different individual models to construct diversity. 
However, in this section, from a unified loss landscape and feature representation distribution perspective, we reveal that the difference between the individuals among ensembles with such diversify strategies may not be as significant as expected, because with same data distribution and optimizer, different individual models' optima can be connected or aligned~\cite{tatro2020optimizing} with a easy permutation on either weights or features.

\textbf{Conjecture 1} (Feature Transformation Alignment). If there is linear mode connectivity between the two models, then based on Optimal Transport theory, both ID and OOD samples' penultimate layer feature maps generated by the models can be aligned by an affine transformation with a very small number of training sample features calibrated.

If the individual models in the ensemble fall into the same or perturbed-symmetric loss landscape basin after the SGD optimization, then these individual models perform similarly in the penultimate layer feature representation.
            This similarity of the features fails to provide much representation diversity in an emsemble, therefore, leads to limited improvement in OOD detection performance.

\subsection{Measuring Mode Connectivity with Loss Barrier}
\label{para:3.2}
Git Re-basin~\cite{ainsworth2023git} gives a view that two models trained with SGD can be trapped in a permeated-symmetric basin and their behavior is similar.
    Given such two models, their parameters can be calibrated after a simple affine permutation.
        As shown in Fig.~\ref{fig:rebasin}~(a), if we simply train the models from two different initializations, they can easily end up in the same or symmetric loss basin since their objectives are the same. When training with different subsets, the loss landscape is not very different because the samples still obey the assumption of independent identical distribution. 
    Following Git Re-basin, we also try to apply the same perturbation (STE matching~\cite{ainsworth2023git}) to models trained on different hypotheses to find whether there is linear mode connectivity between the two models.
        We calculate the loss barrier of the model after the perturbation, which has the following definition~\cite{frankle2020linear}. To make the loss uniform, we define the loss function for the current parameters on the target task of the model rather than on the pre-training task.
\begin{equation}
 \resizebox{0.91\hsize}{!}{$
\begin{split}
    LossBar(\Theta_A,\Theta_B) = &max_{\alpha \in [0,1]} \mathcal{L}((1-\alpha) \Theta_A + \alpha \Theta_B)\\&-\frac{1}{2}(\mathcal{L}(\Theta_A)+\mathcal{L}(\Theta_B)),
\end{split}$}
\end{equation}
where $\Theta_A,\Theta_B$ are trained parameters, and $\mathcal{L}(\cdot)$ is loss function on the target task. Related concept is in Appendix~\ref{A4.1.1}.

\begin{table}
\vspace{-3mm}
  \caption{Self-Coupling Index and Loss Barrier for models trained under different initialization and training strategies. The model structure is ResNet18. The loss is measured on CIFAR10. }
  \label{sample-table}
    \resizebox{3.25in}{!}{
  \centering
  \begin{tabular}{lcccccc}
    \toprule
   & \multicolumn{2}{c}{SupCE}   & \multicolumn{2}{c}{SupCon}              & \multicolumn{2}{c}{SimCLR}   \\
    \cmidrule(r){2-3}  \cmidrule(r){4-5}  \cmidrule(r){6-7} 
         & SCI     & Loss Barrier & SCI     & Loss Barrier & SCI        & Loss Barrier  \\
    \midrule
    SupCE & \textbf{0.861}  & \textbf{0.1047}   & 0.203 & 2.1179 &0.091 & 2.3557 \\
    SupCon  & 0.214 & 2.1495  & \textbf{0.877} & \textbf{0.0993} & 0.107 & 2.3456   \\
    SimCLR & 0.094   & 2.3447  & 0.089 & 2.3155 & \textbf{0.834} &\textbf{0.1579} \\
    \bottomrule
  \end{tabular}}
  \label{tab:diff}
  \vspace{-4mm}
\end{table}

\subsection{Measuring Feature Differences with Self-Coupling}

To test Conjecture 1 from the feature representation level, we denote the penultimate layer features of the two models based on Hypothesis $H_1$ and $H_2$ for sample $x_i$ as $g_{H_1}(x_i)$ and $g_{H_2}(x_i)$, respectively. Taking a constant $N$ such that $|\mathcal{Y}| < N <dim\ll |\mathcal{X}|$, we randomly select $N$ samples from the training set, denoted as $x_{(1)},x_{(2)},\cdots,x_{(N)}$. 
    Then, according to these N samples, a linear transformation matrix $\mathbf{A} \in \mathbb{R}^{dim\times dim}$ and a deviation vector $\vec b \in \mathbb{R}^{dim}$ are calculated such that $g_{H_1}(x_{(i)}) = \mathbf Ag_{H_2}(x_{(i)}) +\vec b,\ i = 1\cdots N$.  
        For all test samples (ID and OOD), we generate $g_{H_2}(x_j)$'s counterpart representation $g_{H_2}'(x_j) = \mathbf{A} g_{H_2}(x_j) + \vec b, \ j \in D_{ID}\cup D_{OOD}$. 
            If Conjecture 1 holds, then $g_{H_1}(x_i)$ will have a high probability of corresponding to $g_{H_2}'(x_i)$ with respect to index $i$ when we establish the Optimal Transport between the distributions of the penultimate layer features $g_{H_1}(D)$ and $g_{H_2}'(D)$ of the test sample set $D = D_{ID}\, \text{or}\, D_{OOD}$.

Optimal Transport outputs a deterministic mapping for any pair of continuous distributions where the mass of distribution $g_{H_1}(\mathbf{x})$ is pushed forward to another distribution $g_{H_2}'(\mathbf{x})$. 
    For a given sample set $D$, we define this mapping by the Sinkhorn distance~\cite{cuturi2013sinkhorn} as a coupling matrix $\mathbf{P}_{H_1,H_2}$, which describes how much probability mass from one point in support of $g_{H_1}(D)$ is assigned to a point in support of $g_{H_2}'(D)$. The calculation and constraint of the coupling matrix are shown in Appendix \ref{A3}.
        The diagonal of the coupling matrix $\mathbf{P}_{H_1,H_2}$ represents the sample's own-to-self assignment.      The diagonal highlighting of the coupling matrix indicates that the difference in the feature representation of the two models is small for any sample. 
            Hence, we define a Self-Coupling Index between two models which indicates the degree of consistency in the feature representations of the models. 

\textbf{Definition 1} (Self-Coupling Index) Given two models trained on the hypotheses $H_1$ and $H_2$, the self-coupling index $\mathcal{C}_{H_1,H_2} \in [0,1]$ between the models is defined as:
\begin{equation}
    \mathcal{C}_{H_1,H_2} = \frac{|\mathcal{X}|}{k}tr(\mathbf{P}_{H_1,H_2,top_k}).
\end{equation}

As shown in Fig.~\ref{fig:barrier}(b) H1-H3, we train two ResNet~18~\cite{he2016deep} models with different initialization $\Theta_1$ and $\Theta_2$, we observe highlighting on the diagonal of the coupling matrix, and the self-coupling index $\mathcal{C}_{\Theta_1,\Theta_2} = 0.853$, indicating the difference in penultimate layer feature is minimal.





\subsection{Different Initializations and Dataset Partition Provide Limited Diversity}
If a low loss barrier can be generated between models by perturbing the weights, the result is that the feature representations generated by these models can also be aligned by a simple transformation.
As shown in Fig.~\ref{fig:barrier}, we trained multiple models based on different hypotheses, and we found a significant correlation between linear mode connectivity and coupling matrix. When the loss barrier is large, we find that the corresponding two models generate both ID and OOD features with a more confusing coupling matrix, implying that the difference between the features generated by the models is significant.

The model pairs trained on different sets of hypotheses differ significantly in terms of the loss barrier.
    In Fig.~\ref{fig:barrier}, the two models trained based on hypotheses H1 and H3 demonstrate that, in agreement with mode connectivity theory, the differences introduced by the different initializations are easily eliminated, i.e., the variability they provide is very small.
        Surprisingly, the two models trained based on hypotheses H4 and H5 use different initializations and two independently sampled subsets of the training set, but both their ID and OOD features are also still highly self-coupled.
            Thus, different model initialization and data partitioning fail to provide sufficient feature representation diversity.

\begin{table*}
  \caption{\textbf{Results on CIFAR10 Benchmark.} Comparison with competitive OOD detection methods. All results are in percentages. Some of the baseline results are from~\cite{sun2022out}.}
  \resizebox{\textwidth}{!}{
  \centering    
  \tabcolsep 1pt
  \begin{tabular}{lcccccccccccc}
    \toprule

     \multicolumn{13}{c}{\textbf{OOD Dataset}}\\
   & \multicolumn{2}{c}{\textbf{SVHN}}   & \multicolumn{2}{c}{\textbf{LSUN}}              & \multicolumn{2}{c}{\textbf{iSUN}}      & \multicolumn{2}{c}{\textbf{Texture}}     & \multicolumn{2}{c}{\textbf{Places365}}     & \multicolumn{2}{c}{\textbf{Average}}  \\
    \cmidrule(r){2-3}  \cmidrule(r){4-5}  \cmidrule(r){6-7}  \cmidrule(r){8-9}  \cmidrule(r){10-11}  \cmidrule(r){12-13}
      Methods   & FPR95     & AUROC & FPR95     & AUROC & FPR95     & AUROC       & FPR95     & AUROC & FPR95     & AUROC & FPR95     & AUROC \\
    \midrule
    ODIN~\cite{liang2018enhancing}  & 20.93 &95.55 &7.26 &98.53 &33.17 &94.65 &56.40 &86.21& 63.04 &86.57 &36.16 &92.30\\ 
    SSD+~\cite{sehwag2021ssd} &1.51 &99.68 &6.09 &98.48& 33.60 &95.16 &12.98 &97.70 &28.41 &94.72 &16.52 &97.15 \\
    CSI~\cite{tack2020csi} &37.38 &94.69 &5.88 &98.86 &10.36 &98.01 &28.85 &94.87 &38.31 &93.04 &24.16 &95.89\\
    MSP~\cite{hendrycks2017baseline} &59.66 &91.25 &45.21 &93.80 &54.57 &92.12 &66.45 &88.50 &62.46 &88.64 &57.67 &90.86\\
    Mahalanobis~\cite{lee2018simple} &9.24 &97.80 &67.73 &73.61 &\textbf{6.02} &\textbf{98.63} &23.21 &92.91 &83.50 &69.56 &37.94 &86.50\\
    Energy~\cite{liu2020energy} &54.41 &91.22 &10.19 &98.05 &27.52 &95.59 &55.23 &89.37 &42.77 &91.02 &38.02 &93.05\\
    KNN~\cite{sun2022out} &24.53 &95.69 &25.29 &95.96 &25.55 &95.26 &27.57 &94.71 &50.90 &89.14 &30.77 &94.15\\
    KNN+\cite{sun2022out} &2.42 &99.52 &1.78 &99.48 &20.06 &96.74 &8.09 &98.56 &23.02 &95.36 &11.07 &97.93\\
    \midrule
    \textbf{MC Ens.+MSP} & 37.49 & 92.22 & 33.96 & 94.96 & 43.96 & 92.21 & 43.68 &92.43 & 39.68 & 90.15 & 39.75 & 92.39\\
    \textbf{MC Ens.+Mahala.} &2.09 & 99.48 & 43.35 & 93.79 & 21.59&94.77&14.31& 94.68 & 27.68 & 89.88 & 21.80 & 94.52\\
    \textbf{MC Ens.+Energy} &34.99 & 92.58 & 6.05& 99.05& 17.96 &96.59 &23.97 &91.92 & 33.02 & 92.37 & 23.20 & 94.50\\
    \textbf{MC Ens.+KNN} & \textbf{1.35} & \textbf{99.70} & \textbf{1.45} & \textbf{99.80} & 7.88 &98.09 & \textbf{4.07} & \textbf{99.05} & \textbf{13.19} & \textbf{97.01} & \textbf{5.58} & \textbf{98.73}\\
    \bottomrule
  \end{tabular}}
  \label{tab:cifar}
\vspace{-4mm}
\end{table*}

\vspace{3mm}
\section{Improving Diversity with Multi-Comprehension Ensemble}
\label{S4}

\subsection{Exploring Multi-Comprehension via Training Tasks}

   

\textbf{Conjecture 2} (Multi-Comprehension) Using distinct pre-training tasks helps to improve the loss barrier between the models and thus helps to improve the ensemble diversity.
    
Conjecture 2 means different comprehension is developed through different training tasks. 
        As shown in Table~\ref{tab:comprehension}, different training criteria ~(tasks) are corresponding to different comprehensions to input data.     
            This is because when designing different training criteria, the corresponding objectives are different so that the trained individual models will have a different comprehension of the inputs, which is difficult to translate into each other by simple perturbations at the parameter or feature representation level. 
                Thus, the diversity provided by multiple training tasks is much more significant.

    Based on the analysis in Section~\ref{para:3.2}, enlarging the loss barrier between models is one of the keys to feature representation diversity.
     An intuitive way to enlarge the loss barrier is to train the individuals using different training tasks, i.e., train the weights using different losses. 
    To generate a loss barrier by having the parameters arrive in completely different basins after training, we can train the model on a completely different loss landscape defined by the loss function. 
        As shown in Fig.~\ref{fig:rebasin}~(b), we can use other training criteria to make the parameters go in the other direction during the stochastic gradient descent, which finally fall into the symmetric unreachable basin and produce a total different feature representation with the original task training.

We find that models trained with different tasks are more likely to have a smaller Self-Coupling Index.
        As shown in Table~\ref{tab:diff}, we verified the Self-Coupling Index between a fraction of three commonly used training criteria on ResNet-18, and their loss barrier on the CIFAR10 classification task. More Self-Coupling Index on different models, datasets, and training tasks can be found in Appendix \ref{A4}.

\subsection{Building Multi-Comprehension Ensemble}

Given a candidate pool with $N$ individual model hypotheses $\mathbb{H} = \{H_1, \cdots, H_N \}$, we select $M$ of them to form an ensemble. We call this ensemble with individual models trained on different hypotheses a Multi-Comprehension Ensemble~(MC Ensemble).

\textbf{Self-Coupling Index guided model selection:} When selecting individual models in an ensemble, we need to consider both the performance of ID samples and the feature diversity. We use the loss of the model on the ID dataset to measure its ID performance and the Self-Coupling Index (as in Table~\ref{tab:diff},~\ref{tab:sci-res18-cifar},~\ref{tab:sci-res50-cifar},~\ref{tab:sci-res50-im}, and~\ref{tab:sci-vit}) to measure feature diversity. Thereby, the problem of constructing an ensemble can be transformed into the following minimization problem:
\begin{equation} 
\resizebox{0.91\hsize}{!}{$
    \min_{H_1,\cdots,H_M \in \mathbb{H}} \frac{1}{M}\sum_{i=1}^{M}\mathcal{L}_{CE}(H_i) + \lambda \frac{1}{M(M-1)}\sum_{i \neq j}\mathcal{C}_{H_i,H_j}, $}
\end{equation}
where $\mathcal{L}_{CE}(H_i)$ indicates the loss of hypothesis $H_i$ in the main task, $\lambda$ is an adjustable parameter.

We construct an instantiated MC~Ensemble with three individuals trained on cross-entropy, SimCLR, and SupCon loss respectively. All three individual models are trained on the whole dataset, given different initializations. 


\vspace{-1mm}
\section{Experiment}

\textbf{Datasets:} We evaluate Multi-Comprehensive Ensemble on two benchmarks: CIFAR Benchmark and ImageNet Benchmark. In CIFAR Benchmark, CIFAR10~\cite{krizhevsky2009learning} is used as ID dataset, and SVHN~\cite{netzer2011reading}, iSUN~\cite{xu2015turkergaze}, LSUN~\cite{yu2015lsun}, Texture~\cite{cimpoi2014describing} and Places365~\cite{zhou2017places} are used as OOD datasets. Furthermore, CIFAR100~\cite{krizhevsky2009learning} is also tested as OOD to evaluate near OOD performance. In ImageNet Benchmark, Imagenet-1K~\cite{deng2009imagenet} is used as the ID dataset, and Places365~\cite{zhou2017places}, SUN~\cite{xiao2010sun}, Texture~\cite{cimpoi2014describing} and iNaturalist~\cite{van2018inaturalist} are used as OOD datasets. We use 4 NVIDIA A100s for model training.

\begin{table*}
  \caption{\textbf{Comparison with naive ensemble.} Models in naive ensemble are trained from different weight initialization. Models in 3$\times$SupCE* are trained with independently-sampled 80\% training set. Scoring method is KNN. All results are in percentages. Some of the baseline results are from~\cite{sun2022out}.}

  \centering    
  \tabcolsep 1pt
  \begin{tabular}{lccccccccccccc}
    \toprule

     \multicolumn{13}{c}{\textbf{OOD Dataset}}\\
   & \multicolumn{2}{c}{\textbf{SVHN}}   & \multicolumn{2}{c}{\textbf{LSUN}}              & \multicolumn{2}{c}{\textbf{iSUN}}      & \multicolumn{2}{c}{\textbf{Texture}}     & \multicolumn{2}{c}{\textbf{Places365}}     & \multicolumn{2}{c}{\textbf{Average}} & \textbf{SCI} \\
    \cmidrule(r){2-3}  \cmidrule(r){4-5}  \cmidrule(r){6-7}  \cmidrule(r){8-9}  \cmidrule(r){10-11}  \cmidrule(r){12-13}
      Methods   & FPR95     & AUROC & FPR95     & AUROC & FPR95     & AUROC       & FPR95     & AUROC & FPR95     & AUROC & FPR95     & AUROC & ~ \\
    \midrule
    \multicolumn{13}{c}{\textbf{Single Model}}\\
    SupCE  &24.53 &95.69 &25.29 &95.96 &25.55 &95.26 &27.57 &94.71 &50.90 &89.14 &30.77 &94.15 & ~\\ 
    SimCLR &41.69 &92.07 &29.68 &93.44& 43.60 &91.60 &32.98 &92.77 &38.41 &91.72 &29.27 &94.62 &~\\
    SupCon &2.42 &99.52 &1.78 &99.48 &20.06 &96.74 &8.09 &98.56 &23.02 &95.36 &11.07 &97.93&~\\
    \multicolumn{13}{c}{\textbf{Naive Ensemble}}\\
    3$\times$SupCE &21.39 &95.89 &27.15 &94.99 &23.55 &95.37 &19.93 &96.35 &46.88 &90.21 &27.78 &94.56 & 0.868\\
    3$\times$SimCLR &47.48 &88.08 &43.99 &89.98 &36.02 &93.01 &24.24 &92.75 &43.50 &89.39 &39.04 &90.64 & 0.841\\
    3$\times$SupCon &2.21 &99.51 &1.88 &99.44 &12.06 &97.74 &7.19 &98.66 &23.37 &95.02 & 10.34 &98.07 & 0.835\\
    3$\times$SupCE* &52.37 &87.06 &39.41 &90.34 &45.55 &88.27 &48.33 &87.71 &68.90 &74.22 &50.91 &85.52 & 0.834\\
    \midrule
    \textbf{MC Ens.} & \textbf{1.35} & \textbf{99.70} & \textbf{1.45} & \textbf{99.80} &\textbf{7.88} &\textbf{98.09} & \textbf{4.07} & \textbf{99.05} & \textbf{13.19} & \textbf{97.01} & \textbf{5.58} & \textbf{98.73} & \textbf{0.134}\\
    \bottomrule
  \end{tabular}
  \label{tab:cifar2}
\vspace{-1mm}
\end{table*}

\vspace{2mm}
\textbf{Metrics:}
We evaluate OOD detection methods on two standard metrics following common practice~\cite{hendrycks2017baseline}: (1)~AUROC: the area under the receiving operating curve; AUROC measures the model's ability to distinguish between positive and negative samples. It plots the true positive rate (TPR) against the false positive rate (FPR) at various classification thresholds. (2)~FPR@TPR95~(FPR95): It measures the rate at which the model falsely identifies OOD samples as ID samples while maintaining a true positive rate of 95\% for ID samples. A low FPR95 is desirable as it indicates that the model is able to accurately identify OOD samples without flagging too many ID samples as OOD.

\vspace{1.5mm}
\textbf{Scoring methods:} Since our approach explores diversity at the feature representation level, it can be combined with a variety of post-hoc OOD detection scoring metrics based on feature representation. We consider the following mostly-used scoring metrics: 
    (1) MSP~\cite{hendrycks2017baseline},
    (2) Mahalanobis distance~\cite{lee2018simple},
    (3) Energy~\cite{liu2020energy},
    (4) KNN~\cite{sun2022out}. 
These methods work on the premise that ID and OOD feature representations need to be distinguishable. A detailed description of these methods can be found in Appendix \ref{A4.1}.

\subsection{CIFAR10 Benchmark}

\textbf{Training details:} We use ResNet-18 as the backbone of individual models for CIFAR10 benchmark. The number of individual models $M$ in the ensemble is set to 3. We train \textit{SupCE} model with the cross-entropy loss with SGD for 500 epochs, with a batch size of 512. The learning rate starts at 0.5 with a cosine annealing schedule~\cite{loshchilov2017sgdr}. The \textit{SimCLR} model and \textit{SupCon} model are trained following the original setting of \cite{chen2020simple} and~\cite{khosla2020supervised} separately.  Results on ResNet-50 are presented in Appendix \ref{A5}.

\textbf{MC Ensemble's outstanding performance: }We present our OOD detection performance in Table~\ref{tab:cifar}. 
    We compare our results with several baseline models, including MSP~\cite{hendrycks2017baseline}, ODIN~\cite{liang2018enhancing}, Mahalanobis Distance~\cite{lee2018simple}, Energy~\cite{liu2020energy}, KNN~\cite{sun2022out}, CSI~\cite{tack2020csi}, SSD+~\cite{sehwag2021ssd} and KNN+~\cite{sun2022out}. 
Among them, CSI, SSD+, and KNN+ are with contrastive training. When combined with the KNN scoring method, MC Ensemble outperforms other methods on four datasets and averages. 
    Further, MC Ensemble, when combined with MSP, Mahalanobis Distance, Energy, and KNN, outperforms the OOD detection performance of the original single training model~(w/ or w/o contrastive training) under these scoring methods, except on iSUN dataset compared with Mahalanobis distance. SOTA comparison is in Table~\ref{tab:sota}.

\begin{table*}[h]
  \caption{ Comparison with state-of-the-art OOD detection methods. All results are in percentages. *: Outlier Exposure based model.}
  \resizebox{\textwidth}{!}{
  \centering    
  \tabcolsep 1pt
  \begin{tabular}{lcccccccccccc}
    \toprule

     \multicolumn{13}{c}{\textbf{OOD Dataset}}\\
   & \multicolumn{2}{c}{\textbf{SVHN}}   & \multicolumn{2}{c}{\textbf{LSUN}}              & \multicolumn{2}{c}{\textbf{iSUN}}      & \multicolumn{2}{c}{\textbf{Texture}}     & \multicolumn{2}{c}{\textbf{Places365}}     & \multicolumn{2}{c}{\textbf{Average}}  \\
    \cmidrule(r){2-3}  \cmidrule(r){4-5}  \cmidrule(r){6-7}  \cmidrule(r){8-9}  \cmidrule(r){10-11}  \cmidrule(r){12-13}
      Methods   & FPR95     & AUROC & FPR95     & AUROC & FPR95     & AUROC       & FPR95     & AUROC & FPR95     & AUROC & FPR95     & AUROC \\
    \midrule
    FeatureNorm~\cite{yu2023block} & 7.13  & 98.65 & 27.08 & 95.25 & 26.02 & 95.38 & 31.18 & 92.31 & 62.54 & 84.62 & 30.79 & 93.24 \\
DOE*~\cite{wang2023outofdistribution}        & 2.65  & 99.36 & 0     & 99.89 & 0.75  & 99.67 & 7.25  & 98.47 & 15.1  & 96.53 & 5.15  & 98.78 \\
CIDER~\cite{ming2023how}       & 3.04  & 99.5  & 4.1   & 99.14 & 15.94 & 97.1  & 13.19 & 97.3  & 26.6  & 94.64 & 12.57 & 97.55 \\
SHE~\cite{zhang2023outofdistribution}         & 5.87  & 98.74 & 6.67  & 98.42 & 4.16  & 98.85 &       &       & 6.31  & 98.7  &       &       \\
DICE~\cite{sun2022dice}       & 25.99 & 95.9  & 3.91  & 99.2  & 4.36  & 99.14 & 41.9  & 88.18 & 48.59 & 89.11 & 24.95 & 94.3  \\
ASH-S~\cite{djurisic2023extremely}       & 6.51  & 98.56 & 4.96  & 98.92 & 5.17  & 98.9  & 24.34 & 95.09 & 48.45 & 88.31 & 17.89 & 95.96 \\
\textbf{MC Ens.}       & 1.35  & 99.7  & 1.45  & 99.8  & 7.88  & 98.09 & 4.07  & 99.05 & 13.19 & 97.01 & 5.58  & 98.73\\
    \bottomrule
  \end{tabular}}
  \label{tab:sota}

\end{table*}
\begin{table}

  \caption{\textbf{Results on CIFAR10 vs CIFAR100.}  }

  \centering    
  \tabcolsep 14pt
  \begin{tabular}{lcc}
    \toprule     
      Methods   & FPR95     & AUROC  \\
    \midrule
     \multicolumn{3}{c}{\textbf{Single Model}}\\
    SupCE  & 56.76 &88.74 \\ 
    SimCLR &62.38 &89.97 \\
    SupCon &37.42 &92.56 \\
    \multicolumn{3}{c}{\textbf{Naive Ensemble}}\\
    3$\times$SupCE &53.41 &89.22 \\
    3$\times$SimCLR &68.53 &88.44 \\
    3$\times$SupCon &36.72 &92.52 \\
    \midrule
    \textbf{MC Ens.} & \textbf{23.35} & \textbf{94.51} \\
    \bottomrule
  \end{tabular}\vspace{-6 mm}
  \label{tab:cifar100}
\end{table}
\begin{table*}
  \caption{\textbf{Results on ImageNet Benchmark.}  All results are in percentages. Scoring method is KNN.}

  \centering    
  \tabcolsep 1.5pt
  \begin{tabular}{lcccccccccc}
    \toprule

     \multicolumn{11}{c}{\textbf{OOD Dataset}}\\
   & \multicolumn{2}{c}{\textbf{iNaturalist}}   & \multicolumn{2}{c}{\textbf{SUN}}              & \multicolumn{2}{c}{\textbf{Places}}      & \multicolumn{2}{c}{\textbf{Textures}}   & \multicolumn{2}{c}{\textbf{Average}}  \\
    \cmidrule(r){2-3}  \cmidrule(r){4-5}  \cmidrule(r){6-7}  \cmidrule(r){8-9}  \cmidrule(r){10-11}  
      Methods   & FPR95     & AUROC & FPR95     & AUROC & FPR95     & AUROC       & FPR95     & AUROC & FPR95     & AUROC  \\
    \midrule
    \multicolumn{11}{c}{\textbf{Single Model}}\\
    SupCE  &59.00 & 86.47 & 68.82 & 80.72 & 76.28 & 75.76 & 11.77 & 97.07 & 53.97 & 85.01\\ 
    SimCLR & 49.88 & 88.34 & 78.62 & 79.57 & 63.65 & 82.35 & 13.87 & 96.33 & 51.51 & 86.65 \\
    SupCon &30.18 & 94.89 & 48.99 & 88.63 & 59.15 & 84.71 & 15.55 & 95.40 & 38.47 & 90.91\\
    \multicolumn{11}{c}{\textbf{Naive Ensemble}}\\
    3$\times$SupCE  &53.32 & 87.95 & 58.25 & 82.98 & 56.28 & 81.01 & 17.71 & 94.31 & 46.39 & 86.56\\
    3$\times$SimCLR & 46.37 & 88.31 & 77.34 & 80.39 & 64.88 & 82.64 & 15.97 & 95.55 & 51.14 & 86.72 \\
    3$\times$SupCon & 28.93 & 95.21 & \textbf{38.69} & \textbf{91.32} & 59.66 & 84.69 & 15.41 & 95.36 & 35.67 & 91.64\\
    \midrule
    \textbf{MC Ens.} & \textbf{15.39} & \textbf{96.78} & 42.97 & 90.35 &\textbf{54.89} &\textbf{87.34} & \textbf{9.54} & \textbf{97.77}  & \textbf{30.69} & \textbf{93.06}\\
    \bottomrule
  \end{tabular}
  \vspace{-4 mm}
  \label{tab:imagenet}

\end{table*}
\textbf{MC Ensemble leverages the diversity of feature representation:} We compare our MC Ensemble with other ensemble strategies.
    As shown in Table~\ref{tab:cifar2}, we make the naive deep ensemble whose individual models share the same training criterion but are with different weight initializations. 
        Compared with single models, the naive ensemble indeed improves the OOD detection performance when using \textit{SupCE} and \textit{SupCon} training. 
            However, the improvement is limited while we also notice a decrease when we conduct a self-supervised \textit{SimCLR} ensemble.
        When training individual models with the partial dataset, the OOD detection performance drop quickly, even with an ensemble. 
            We argue that this is because OOD detection performance is positively correlated with ID performance; training with a partial dataset will cause the degradation of the model's cognition capacity. 

        MC Ensemble outperforms all the naive ensembles with all scoring methods. This shows that the ensemble's multi-comprehension of the data, i.e., the feature representation diversity with multiple training tasks, brings a significant improvement in OOD performance.        
            Table~\ref{tab:cifar2} shows the results of the KNN scoring method. Results on more scoring methods are presented in Appendix \ref{A6}. 

\vspace{3mm}
\textbf{NearOOD setting:} Near OOD samples are similar to the training data, but still different enough to be considered OOD. We evaluate MC Ensemble near OOD performance on the CIFAR10-vs-CIFAR100 task, which considers CIFAR100 as an OOD dataset. 
    As shown in Table~\ref{tab:cifar100}, consistent with the previous CIFAR benchmark experiments, the naive ensemble does not provide significant OOD performance improvement in the near OOD setting either. 
        Compared with the naive ensemble with 3 models trained with cross-entropy loss and different initializations, MC Ensemble reduces the FPR95 by 30.05\% and improves the AUROC by 5.29\%. 
MC Ensemble outperforming 3$\times$SupCon ensemble indicates that the OOD detection performance improvement is not only gained from supervised contrastive training but also multi-comprehension feature representation diversity.

\subsection{ImageNet Benchmark}
\textbf{Training details:} Following~\cite{sun2022out}, we use ResNet-50 as the backbone of individual models for the ImageNet benchmark. The models are trained on ImageNet-1k~\cite{deng2009imagenet} with resolution 224$\times$224. For \textit{SupCE} model, we import the model from torchvision~\cite{paszke2019pytorch}. 
    The \textit{SimCLR} and \textit{SupCon} models are trained following the original setting in \cite{chen2020simple} and~\cite{khosla2020supervised} separately. 
Results of ViT-B~\cite{dosovitskiy2020image} MC Ensemble trained with cross-entropy, MOCO~v3~\cite{chen2021empirical} and MAE~\cite{he2022masked} are presented in Appendix \ref{A7}.

\textbf{MC Ensemble achieves outstanding performance in large-scale task:} As shown in Table~\ref{tab:imagenet}, consistent with the CIFAR10 benchmark, MC Ensemble outperforms all the naive ensembles on all the OOD datasets except 3$\times$SupCon on SUN dataset. This is most likely because the gap between the OOD detection performance of supervised contrastive training on the ImageNet benchmark and the other two is too large. A comparison with other baseline methods is presented in Appendix \ref{A8}.

\section{Ablation Study}
\label{sec:6}

\textbf{Significance of supervised contrastive training:}
As noticed in~\cite{sun2022out}, supervised contrastive training provides feature representation that is helpful to OOD detection performance. We verify this in Table~\ref{tab:cifar2}. 
    However, we argue that \textit{SupCon} training is not the only contributor to OOD performance improvement in MC Ensemble. 
    As shown in Table~\ref{tab:ablation}, 2$\times$SupCon ensemble can not beat either SupCE+SupCon or SimCLR+SupCon ensemble, indicating that multi-comprehension feature diversity also contributes to OOD detection performance.
\begin{table}

  \caption{\textbf{Ablation Study.} ResNet-62 contains 4 more blocks compared to ResNet-50.}
  \centering    
  \tabcolsep 6pt
  \vskip 0.15in
  \begin{tabular}{lccc}
    \toprule     
      Methods  &\#Params. & FPR95     & AUROC  \\
    \midrule
    2$\times$SupCon & 18.26 & 12.34 &97.09\\
    SupCE+SimCLR & 18.26 & 24.98 & 95.31 \\
    SupCE+SupCon & 18.26 & 9.37 & 97.95 \\
    SimCLR+SupCon & 18.26 & 9.42 & 97.81 \\
    ResNet-62 & 27.50 & 23.79 & 95.34\\
    3$\times$SupCE & 27.39 & 27.78 & 94.56 \\
    \midrule
    \textbf{MC Ens.} & 27.39 & \textbf{5.58} & \textbf{98.73} \\
    \textbf{Distillation} & 9.13 & 8.17 & 98.13\\
    \bottomrule
  \end{tabular}\vspace{-6mm}
  \label{tab:ablation}
\end{table}

\textbf{Comparison with a single model with the same scale:} \cite{abe2022deep} points out that an ensemble has similar performance on OOD to a single model of similar size. In Table~\ref{tab:ablation}, We confirm that this conclusion holds for the naive ensemble by comparing ResNet-62 with the ensemble of 3 ResNet-18s. However, MC Ensemble still significantly outperforms larger single models, indicating that model scale is not the only contributor to MC Ensemble's performance.

\textbf{Combinability with distillation:}
Despite its effectiveness, the use of an ensemble can be limited by the high computational expenses it incurs, making it impractical for certain applications. 
    For its computational overhead, knowledge distillation is an effective method for ensemble model compression.
        To verify whether MC Ensemble's distillable, we directly employ Ensemble Distribution Distillation (EnD$^2$)~\cite{Malinin2020Ensemble} to distill our MC Ensemble to a single model. 
    As shown in Table~\ref{tab:ablation}, with a minimal drop, the distillation model of MC Ensemble maintains a strong OOD detection performance, but the computation overhead in the inference stage is the same as a single individual. 

\vspace{3.5mm}
\section{Ensemble Latency Analysis}
\label{A9}
 The components that contribute to the latency of the OOD detection method usually contain such two phases: (1) Generation of penultimate layer feature representations, (2) Computation of out-of-distribution discriminant score.

In the first phase, latency depends on the inference time of the backbone network. The computational overhead of an MC Ensemble consisting of $M$ individual models is $M\times $ that of a single model (we ignore the overhead of averaging feature, since it is negligibly small compared to neural network models). However, since these individual models are independent of each other, they are model-level parallelizable. Typically, in inference phases with sufficient computation resources, it is possible to achieve a high degree of parallelism of these individual models die on a single device with the help of Nvidia's Multi-Process Scheduling (MPS) or Multi-Instance GPU (MIG) technology. We conduct following experiments to support the above conclusion: we set three individual models (3 x ResNet-18) as multiple independent processes, and utilized Nvidia MIG technology to let these models run simultaneously on the same A100 GPU, with batch size set to 16 (small batch size is more in line with real-world real-time reasoning needs), the latency of MC Ensemble to generate the penultimate layer features was 9.4 ms, while the single ResNet-18 model took 9.3 ms. Furthermore, we tested against bigger models, and MC ensemble can even achieve faster inference than some standalone models of the same size. For example, ResNet-62, which is the same size as MC Ensemble, took 17.0 ms on the same hardware. 

In the second phase, latency depends on the OOD scoring methods. Since MC ensemble can be used with any kind of OOD discriminant scoring metric, it does not have any computational difference in latency compared to other post-hoc OOD detection methods if we use the same scoring metric. With the hyperparameters determined, the computational complexity of the OOD discriminant score depends only on the dimensions of the features, and our strategy of using feature average ensures that the feature dimensions are invariant compared to the original backbone network. Therefore the second phase computational latency depends only on which OOD scoring metric is combined with the MC Ensemble.

\section{Conclusion}


In this paper, we reveal that the different initializations of an original ensemble model do not provide sufficient feature representation diversity, thereby resulting in only minor performance improvements for OOD detection.
By demonstrating that training tasks can induce multiple comprehension of the ensemble model in both feature space similarity angle and loss landscape angle, we propose a method, named MC Ensemble, to enhance the diversity of feature representation, which improves the OOD detection performance of ensemble models.
We validate the excellent performance of MC Ensemble through experimental evaluation on CIFAR10 and ImageNet Benchmark datasets.

\vspace{5mm}
\section*{Impact Statement}

Generally, we believe OOD detection is an important component of AI safety. Enhancing OOD detection impacts the reliability of AI application in autonomous driving, healthcare, and others. The negative impact may be that the large amount of computation of ensemble-based model may cause larger computational resource footprint and carbon footprint. Disscusion on limitations can be found in Appendix \ref{A10}.

\section*{Acknowledgements}
This material is based on research sponsored by the Air Force Research Laboratory (AFRL)
under agreement number FA8750-21-1-1015. The U.S. Government is authorized to reproduce and
distribute reprints for Governmental purposes notwithstanding any copyright notation thereon.
 The views and conclusions contained herein are those of the authors and should not be interpreted as
necessarily representing the official policies or endorsements, either expressed or implied, of the Air
Force Research Laboratory (AFRL) or the U.S. Government.

\bibliographystyle{_bib/icml2023}
\bibliography{_bib/icml2024}

\newpage

\newpage

\appendix
\onecolumn

\section{Appendix}
\subsection{Bias-Variance Decomposition of OOD Detection Ensemble}
\label{A1}
\subsubsection{Bias-Variance Decomposition of OOD Detection Model}
We consider the OOD detection task as a 0-1 classification problem on an open set, where samples sampled from a distribution consistent with the training set $D_{train}$ (ID) are regarded as positive, and samples sampled from outside the distribution (OOD) are regarded as negative (assumed to be sampled from $D_{OOD}$).

The probability that an OOD detector coupled with a neural network trained on hypothesis $H$ regard sample $x$ as positive can be formulated as:
\begin{equation}
    P(\Gamma_H = 1 |x) = E[\mathbf{1}[s(g_H(x))>\tau]] = P(s(g_H(x))>\tau)    
\end{equation}

\begin{equation}
    \Gamma_H(x,\tau) = \left\{\begin{array}{lr}
        1, & \text{if}\ s(g_H(x))>\tau \\
        0, & \text{if}\ s(g_H(x))<\tau
    \end{array}
    \right.
\end{equation}

Denote the ground truth classifier as $\Gamma_T$.
\begin{theorem}
$\Gamma_H$ and $\Gamma_T$ are conditionally independent given target $f$ and a test point $x$.
\end{theorem}
\begin{proof}
$P(\Gamma_T,\Gamma_H|f,x)=P(\Gamma_T|\Gamma_H,f,x)P(\Gamma_H|f,x)=P(\Gamma_T|f,x)P(\Gamma_H|f,x)$. 
\end{proof}The last equality is true because by definition, the ground truth classifier $\Gamma_T$ only depends on the target $f$ and the test point $x$.

Regarded as a 0-1 classification problem, the loss of detector $\Gamma_H$ can be subjected to a bias-variance decomposition~\cite{kohavi1996bias} like Eq.~(\ref{equ: composition}),
\begin{equation}
\begin{split}
    \mathcal{L}(\Gamma_H) &= \frac{1}{2}\sum_x P(x)[\underbrace{\sum_{y=0}^1(P(\Gamma_H=y|x)-P(\Gamma_T = y|x))^2}_{bias^2_H}\\
    &+\underbrace{(1-\sum_{y=0}^1P(\Gamma_H=y|x)^2)}_{variance_H}+\sigma_x^2]
        \label{equ: composition}
\end{split}
\end{equation}

where $T$ means ground truth hypothesis which can be seen as a perfect OOD detector and $\sigma_x^2 = 1- \sum_{y=0}^1P(\Gamma_T=y|x)^2$ is an irreducible error.
\begin{proof}
First we consider only the distribution of the detector output:
\begin{equation*}
\begin{aligned}
        \mathcal{L}(\Gamma_H) &= 1 - \sum_{y=0}^1P(\Gamma_H=\Gamma_T=y) \\
    &= \sum_{y=0}^1-P(\Gamma_H=\Gamma_T=y) + \sum_{y=0}^1P(\Gamma_H=y)P(\Gamma_T=y)\\
    & ~~~~+\sum_{y=0}^1[-P(\Gamma_H=y)P(\Gamma_T=y)+\frac{1}{2}P(\Gamma_T=y)^2+\frac{1}{2}P(\Gamma_H=y)^2]\\
    & ~~~~+[\frac{1}{2}-\frac{1}{2}P(\Gamma_H=y)^2]+[\frac{1}{2}-\frac{1}{2}P(\Gamma_T=y)^2]\\
    &= \sum_{y=0}^1[P(\Gamma_H=y)P(\Gamma_T=y)-P(\Gamma_H=\Gamma_T=y)]\\
    &~~~~+\frac{1}{2}\sum_{y=0}^1(P(\Gamma_H=y)-P(\Gamma_T=y))^2\\
    &~~~~+\frac{1}{2}(1-\sum_{y=0}^1P(\Gamma_H=y)^2)\\
    &~~~~+\frac{1}{2}(1-\sum_{y=0}^1P(\Gamma_T=y)^2)
\end{aligned}
\end{equation*}
Due to the independence between the detector and ground truth, the first term disappears. Now, we consider the conditional probabilities on the data set.

\begin{equation*}
\begin{aligned}
        \mathcal{L}(\Gamma_H) &= 1 - \sum_{x}P(x)\sum_{y=0}^1P(\Gamma_H=\Gamma_T=y|x) \\
    &= \sum_{x}P(x)\frac{1}{2}\sum_{y=0}^1(P(\Gamma_H=y|x)-P(\Gamma_T=y|x))^2 ~~~~~~~~~~~&(bias^2_H)\\
    &~~~~+\sum_{x}P(x)\frac{1}{2}(1-\sum_{y=0}^1P(\Gamma_H=y|x)^2)&(variance_H)\\
    &~~~~+\sum_{x}P(x)\frac{1}{2}(1-\sum_{y=0}^1P(\Gamma_T=y|x)^2)&(\sigma^2_x)
\end{aligned}
\end{equation*}
\end{proof}
\subsubsection{Variance Term Decomposition of OOD Detection Ensemble}
The ensemble is widely used in the deep learning community as a scalable and simple method. 
    The core idea of the ensemble is to exploit the diversity among different models. The variance of an ensemble in Eq.~(\ref{equ: composition}) with $M$ individuals which are based on hypothesis $H_i, i \in \{1, \cdots, M\}$ can be further composed to:
\begin{equation}
\begin{split}
    &variance_{ens} = \frac{1}{M}[\underbrace{\frac{1}{M}\sum_{i=1}^{M}(1-\sum_{y=0}^1P(\Gamma_{H_i}=y|x)^2)}_{E[variance_{H_i}]}\\
    &+\underbrace{\frac{1}{M}\sum_{i=1}^{M}\sum_{j \neq i}(1-\sum_{y=0}^1P(\Gamma_{H_i}=y|x)P(\Gamma_{H_j}=y|x))}_{covariance}],
    \end{split}
\end{equation}
hence, the variance of an ensemble can be bounded by a lower boundary $\frac{1}{M}E[variance_{H_i}]$ and an upper boundary $E[variance_{H_i}]$.

\begin{proof}
For an ensemble model, the variance term can be further decomposed: 
\begin{equation*}
    \begin{aligned}
        variance_{ens} &= 1-\sum_{y=0}^1(\sum_{i=1}^MP(\Gamma_{H_i}=y|x))^2\\
        &= \frac{1}{M^2}\sum_{i=1}^{M}\sum_{j=1}^{M}(1-\sum_{y=0}^1P(\Gamma_{H_i}=y|x)P(\Gamma_{H_j}=y|x))\\
        &=\frac{1}{M}[\frac{1}{M}\sum_{i=1}^{M}(1-\sum_{y=0}^1P(\Gamma_{H_i}=y|x)^2)~~~~&(E[variance_{H_i}])\\
        &~~~~+\frac{1}{M}\sum_{i=1}^{M}\sum_{j \neq i}(1-\sum_{y=0}^1P(\Gamma_{H_i}=y|x)P(\Gamma_{H_j}=y|x))],&(covariance)
    \end{aligned}
\end{equation*}
Therefore,
\begin{equation*}
    \begin{aligned}
\frac{1}{M}E[variance_{H_i}] & = \frac{1}{M}[\frac{1}{M}\sum_{i=1}^{M}(1-\sum_{y=0}^1P(\Gamma_{H_i}=y|x)^2) \\
&\leq \frac{1}{M}[\frac{1}{M}\sum_{i=1}^{M}(1-\sum_{y=0}^1P(\Gamma_{H_i}=y|x)^2)\\
        &~~~~+\frac{1}{M}\sum_{i=1}^{M}\sum_{j \neq i}(1-\sum_{y=0}^1P(\Gamma_{H_i}=y|x)P(\Gamma_{H_j}=y|x))]\\
        &\leq \frac{1}{M}[\frac{1}{M}\sum_{i=1}^{M}(1-\sum_{y=0}^1P(\Gamma_{H_i}=y|x)^2)\\
        &~~~~+\frac{1}{M}\sum_{i=1}^{M}(M-1)(1-\sum_{y=0}^1P(\Gamma_{H_i}=y|x)^2] \\
        & = E[variance_{H_i}],
    \end{aligned}
\end{equation*}
\end{proof}

When the $covariance$ is 0, i.e., all detectors are completely uncorrelated, the variance of the ensemble can reach the lower bound, which is $\frac{1}{M}$ in a single model; 
        while when all models are highly similar, the covariance of the models will become larger and the significance of the ensemble will then diminish.
    Therefore, substantial model diversity can significantly reduce the covariance term, thus improving the ensemble performance.
    
\begin{figure}[h!]
    \centering
    \includegraphics{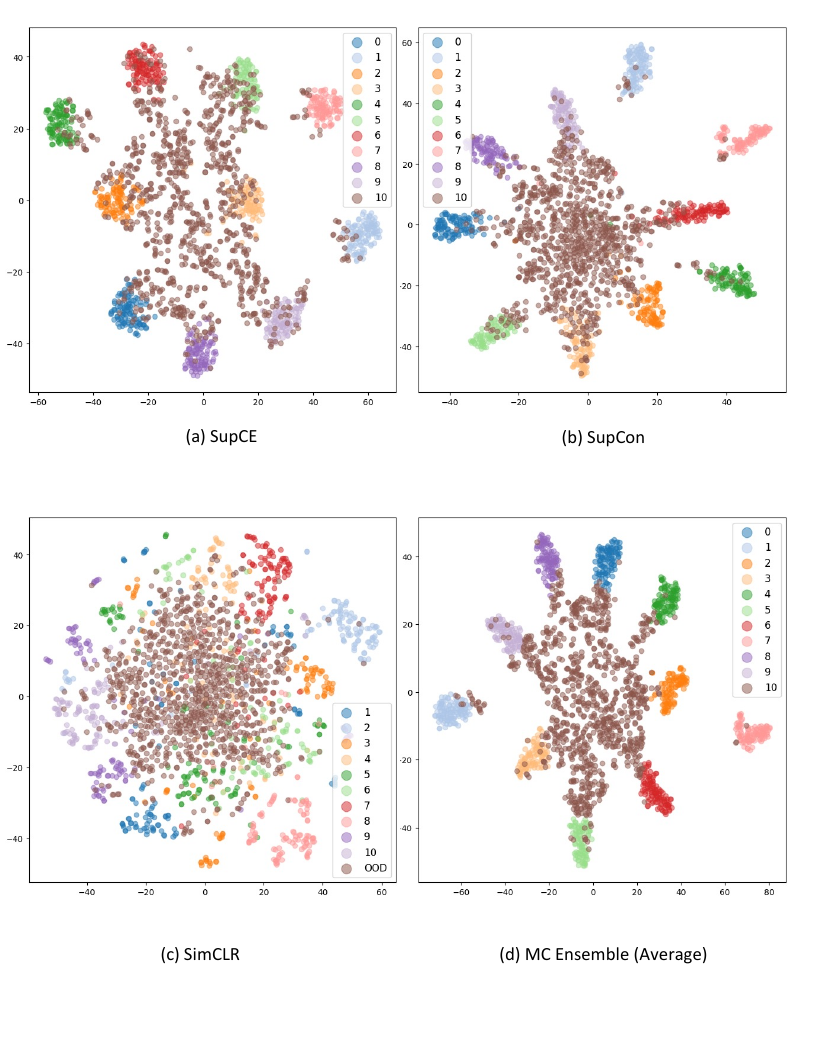}
    \caption{t-SNE visualization of penultimate layer features of SupCE, SupCon, SimCLR, MC Ensemble(average), while class 0-9 are ID classes (CIFAR10) and class 10 is OOD (CIFAR100).}
    \label{fig:feature}
\end{figure} 

\subsection{No Counteraction in Feature-level Ensemble}
\label{A2.1}
Averaging features will not lead to counteraction because the fact that: The feature space is a high-dimensional space, in which arbitrary two vectors are almost orthogonal. This is based on the fact that, for a n-dim space, the angle $\theta$ between any two vectors satisfies: $P(|\theta-\frac{\pi}{2}|\leq m) = 1 - \frac{\int_{0}^{\frac{\pi}{2}-m}\sin^{n-2}\theta d\theta}{\int_{0}^{\frac{\pi}{2}}\sin^{n-2}\theta d\theta}$, where the numerator $\int_{0}^{\frac{\pi}{2}-m}\sin^{n-2}\theta d\theta < (\frac{\pi}{2}-m)\sin^{n-2}(\frac{\pi}{2}-m)$ decreases exponentially with $n$, while the denominator $\int_{0}^{\frac{\pi}{2}}\sin^{n-2}\theta d\theta > \frac{2\sqrt2}{3\sqrt{n-2}}$ decreases no faster than $o(\sqrt n)$. This suggests that any two features are almost orthogonal and that averaging them will not counteract each other. Fig.~\ref{fig:feature} gives an example of feature-level ensemble's t-SNE visualization.

\subsection{Feature-level Ensemble Activation Analysis}
\label{A2}
Suppose $X \sim ESN(\mu,\sigma,\epsilon)$. First, we consider the probability density function of X:
\begin{equation}
  p(x) =  \begin{cases}
          \phi((x-\mu)/(1+\epsilon)\sigma)/\sigma, \qquad\text{if} \quad x<\mu,\\ \phi((x-\mu)/(1-\epsilon)\sigma)/\sigma,\qquad \text{if} \quad x\geq\mu, 
    \end{cases} 
\end{equation}
Therefore, considering activation function, $Z = max(0,X)$, we have expectation: 
\begin{equation}
\begin{aligned}
    \mathbb{E}[Z] &= \mu+\int\limits_{-\infty}^{-\mu}-\frac{\mu}{\sigma}\phi\left(\frac{x}{(1+\epsilon)\sigma}\right)\,\mathrm{d}x+\int\limits_{-\mu}^{0}\frac{x}{\sigma}\phi\left(\frac{x}{(1+\epsilon)\sigma}\right)\,\mathrm{d}x+\int\limits_{0}^{\infty}\frac{x}{\sigma}\phi\left(\frac{x}{(1-\epsilon)\sigma}\right)\,\mathrm{d}x \\
&= \mu-\mu(1+\epsilon)\Phi\left(\frac{-\mu}{(1+\epsilon)\sigma}\right) + (1+\epsilon)\int\limits_{-\mu}^{0}\frac{x}{(1+\epsilon)\sigma}\phi\left(\frac{x}{(1+\epsilon)\sigma}\right)\,\mathrm{d}x +(1-\epsilon)\int\limits_{0}^{\infty}\frac{x}{(1-\epsilon)\sigma}\phi\left(\frac{x}{(1-\epsilon)\sigma}\right)\,\mathrm{d}x \\
&= \mu-\mu(1+\epsilon)\Phi\left(\frac{-\mu}{(1+\epsilon)\sigma}\right) + (1+\epsilon)^2\left[\phi\left(\frac{-\mu}{(1+\epsilon)\sigma}\right)-\phi(0)\right]\sigma + (1-\epsilon)^2\phi(0)\sigma\\
& = \mu\left[1-(1+\epsilon)\Phi\left(\frac{-\mu}{(1+\epsilon)\sigma}\right)\right] +(1+\epsilon)^2\phi\left(\frac{-\mu}{(1+\epsilon)\sigma}\right)\sigma -(1+\epsilon)^2\phi\left(0\right)\sigma +(1-\epsilon)^2\phi\left(0\right)\sigma \\
& = \mu\left[1-(1+\epsilon)\Phi\left(\frac{-\mu}{(1+\epsilon)\sigma}\right)\right] + (1+\epsilon)^2\phi\left(\frac{-\mu}{(1+\epsilon)\sigma}\right)\sigma -4\epsilon\phi(0)\sigma \\
& = \mu\left[1-(1+\epsilon)\Phi\left(\frac{-\mu}{(1+\epsilon)\sigma}\right)\right] + (1+\epsilon)^2\phi\left(\frac{-\mu}{(1+\epsilon)\sigma}\right)\sigma -\frac{4\epsilon\sigma}{\sqrt{2\pi}}, \\
\end{aligned}
\label{expe}
\end{equation}
where $\Phi(\cdot)$ and $\phi(\cdot)$ denote the cdf and pdf of a standard normal distribution separately.

For a single model's activation, substitute $\sigma = \sigma_{\text{in}}$ and $\epsilon = 0$ into Eq.(\ref{expe}), then the expectation of ID activation will be:
\begin{equation}
    \mathbb{E}_{\text{in}}[z_i] = \left[1-\Phi\left(\frac{-\mu}{\sigma_{\text{in}}}\right)\right]\mu+\phi\left(\frac{-\mu}{\sigma_{\text{in}}}\right)\sigma_{\text{in}},
\end{equation}
 For a feature-level ensemble that averages the pre-activation features of $M$ models, the pre-activation feature $x_i \sim \mathcal{N}(\mu,\frac{\sigma_{\text{in}}^2}{M})$, substitute $\sigma = \sigma_{\text{in}}/\sqrt{M}$ and $\epsilon = 0$ into Eq.~(\ref{expe}): 
\begin{equation}
    \mathbb{E}_{\text{in}}[\bar{z_i}] = \left[1-\Phi\left(\frac{-\mu\sqrt{M}}{\sigma_{\text{in}}}\right)\right]\mu+\phi\left(\frac{-\mu\sqrt{M}}{\sigma_{\text{in}}}\right)\frac{\sigma_{\text{in}}}{\sqrt{M}}.
\end{equation}

For OOD data, it can be obtained in the same way that the single model and ensemble activation's expection would be:
\begin{equation}
     \mathbb{E}_{\text {out}}\left[z_i\right]=\mu-(1+\epsilon) \Phi\left(\frac{-\mu}{(1+\epsilon) \sigma_{\text {out}}}\right) \mu\\+(1+\epsilon)^2 \phi\left(\frac{-\mu}{(1+\epsilon) \sigma_{\text {out}}}\right) \cdot \sigma_{\text {out}}-\frac{4 \epsilon}{\sqrt{2 \pi}} \sigma_{\text {out}},
\end{equation}
\begin{equation}
     \mathbb{E}_{\text {out}}\left[\bar{z_i}\right]=\mu-(1+\epsilon) \Phi\left(\frac{-\mu\sqrt{M}}{(1+\epsilon) \sigma_{\text {out}}}\right) \cdot \mu\\+(1+\epsilon)^2 \phi\left(\frac{-\mu\sqrt{M}}{(1+\epsilon) \sigma_{\text {out}}}\right) \frac{\sigma_{\text {out}}}{\sqrt{M}}-\frac{4 \epsilon}{\sqrt{2 \pi M}} \sigma_{\text {out}},
\end{equation}

For ID data, compared with a single model, the average movement of feature averaging ensemble for activation is:
\begin{equation}
     \mathbb{E}_{\text{in}}[\bar{z_i}-z_i] = \mu\left[\Phi\left(\frac{\mu\sqrt{M}}{\sigma_{\text{in}}}\right)-\Phi\left(\frac{\mu}{\sigma_{\text{in}}}\right)\right]+\sigma_{\text{in}}\left[\frac{1}{\sqrt{M}}\phi\left(\frac{\mu\sqrt{M}}{\sigma_{\text{in}}}\right)-\phi\left(\frac{-\mu}{\sigma_{\text{in}}}\right)\right].
\end{equation}

For OOD data, the corresponding movement will be:
\begin{multline}
    \mathbb{E}_{\text {out}}\left[\bar{z_i}-z_i\right] = \frac{4\epsilon\sigma_{\text{out}}}{\sqrt{2\pi}}\left(1-\frac{1}{\sqrt{M}}\right) +(1+\epsilon)\mu\left[\Phi\left(\frac{\mu\sqrt{M}}{(1+\epsilon)\sigma_{\text{out}}}\right)-\Phi\left(\frac{\mu}{(1+\epsilon)\sigma_{\text{out}}}\right)\right]\\ +(1+\epsilon)^2\sigma_{\text{out}}\left[\frac{1}{\sqrt{M}}\phi\left(\frac{\mu\sqrt{M}}{(1+\epsilon)\sigma_{\text{out}}}\right)-\phi\left(\frac{\mu}{(1+\epsilon)\sigma_{\text{out}}}\right)\right].
\end{multline}
    
Under the same chaoticness level ($\sigma_{\text{in}}=\sigma_{\text{out}}=\sigma$), we make a difference between the ID and OOD movement expectation:
\begin{multline}
\mathbb{E}_{\text{out}}[\bar{z_i}-z_i]-\mathbb{E}_{\text{in}}[\bar{z_i}-z_i] =  \underbrace{\mu\left[(1+\epsilon)\Phi\left(\frac{\mu\sqrt{M}}{(1+\epsilon)\sigma}\right)-(1+\epsilon)\Phi\left(\frac{\mu}{(1+\epsilon)\sigma}\right)-\Phi\left(\frac{\mu\sqrt{M}}{\sigma}\right)+\Phi\left(\frac{\mu}{\sigma}\right)\right]}_{(\MakeUppercase{\romannumeral 1})}
\\  +\underbrace{\sigma\left[\frac{(1+\epsilon)^2}{\sqrt{M}}\phi\left(\frac{\mu\sqrt{M}}{(1+\epsilon)\sigma}\right)-(1+\epsilon)^2\phi\left(\frac{\mu}{(1+\epsilon)\sigma}\right)-\frac{1}{\sqrt{M}}\phi\left(\frac{\mu\sqrt{M}}{\sigma}\right)+\phi\left(\frac{\mu}{\sigma}\right)+\frac{4\epsilon}{\sqrt{2\pi}}\left(1-\frac{1}{\sqrt{M}}\right)\right]}_{(\MakeUppercase{\romannumeral 2})}.
\label{eq:outin}
\end{multline}
Next, we prove that $(\MakeUppercase{\romannumeral 1})$ and $(\MakeUppercase{\romannumeral 2})$ $\leq 0$ separately.
For $(\MakeUppercase{\romannumeral 1})$, we have $\mu > 0$. Since the OOD activation is positive-skewed, we have $-1 <\epsilon < 0$.
Let $a = (1+\epsilon) \in (0,1)$, $b = \frac{\mu}{\sigma} > 0$, and $c = \sqrt{M} > 1$, then:
\begin{equation}
    (\MakeUppercase{\romannumeral 1}) = \mu\left[a\Phi\left(\frac{bc}{a}\right)-a\Phi\left(\frac{b}{a}\right)-\Phi(bc)+\Phi(b)\right].
\end{equation}
Let, $T(x,a) = a\Phi(x/a)-\Phi(x)$, then:
\begin{equation}
     (\MakeUppercase{\romannumeral 1}) = \mu[T(bc,a)-T(b,a)]
\end{equation}
Since (\romannumeral 1): $\frac{\partial T}{\partial x} = \phi(x/a)-\phi(x) <0$ when $x>0$ for all $a \in (0,1)$, (\romannumeral 2): $\mu > 0$, and  (\romannumeral 3): $bc>b$, we have:
\begin{equation}
    (\MakeUppercase{\romannumeral 1}) = \mu[T(bc,a)-T(b,a)] < 0.
    \label{eq:I}
\end{equation}

For $(\MakeUppercase{\romannumeral 2})$, we use same symbol system:
\begin{equation}
    (\MakeUppercase{\romannumeral 2}) = \sigma\left[\frac{a^2}{c}\phi\left(\frac{bc}{a}\right)-a^2\phi\left(\frac{b}{a}\right)-\frac{1}{c}\phi(bc)+\phi(b)+\frac{4\epsilon}{\sqrt{2\pi}}\left(1-\frac{1}{c}\right)\right].
\end{equation}
Let $U(x,a) = a^2\phi(x/a)-\phi(x)$, and:

\begin{equation}
    V(x,a,c) := \sigma\left[\frac{1}{c}U(xc,a)-U(x,a)+\frac{4(a-1)}{\sqrt{2\pi}}\left(1-\frac{1}{c}\right)\right] .
    \label{eq:v}
\end{equation}
Then we have:
\begin{equation}
    (\MakeUppercase{\romannumeral 2}) = V(b,a,c).
\end{equation}
We start with the $b = 0$:
\begin{align}
        V(0,a,c) &= \sigma\left[\frac{1}{c}U(0,a)-U(0,a)+\frac{4\epsilon}{\sqrt{2\pi}}\left(1-\frac{1}{c}\right)\right] \notag\\&= \sigma\left[-(1-\frac{1}{c})U(0,a)+\frac{4\epsilon}{\sqrt{2\pi}}\left(1-\frac{1}{c}\right)\right] \notag\\&=
        \sigma\left[-(1-\frac{1}{c})\frac{a^2-1}{\sqrt{2\pi}}+\frac{4(a-1)}{\sqrt{2\pi}}\left(1-\frac{1}{c}\right)\right] \notag\\&=
        \frac{\sigma}{\sqrt{2\pi}}\left(1-\frac{1}{c}\right)[-a^2+4a-3] \notag \\&\leq 0, \quad \text{if}\quad a \in (0,1).
\end{align}
For $x>0$, we have partial derivatives:
\begin{align}
    \frac{\partial V}{\partial c} &= \sigma\left[-\frac{1}{c^2}U(cx,a)+\frac{1}{c}\frac{\partial U(cx,a)}{\partial c} +\frac{4(a-1)}{\sqrt{2\pi}}\frac{1}{c^2}\right]\notag\\&= \frac{\sigma}{\sqrt{2\pi}}\left[-\frac{1}{c^2}\left(a^2e^{-\frac{c^2x^2}{2a^2}}-e^{-\frac{c^2x^2}{2}}\right)+\frac{1}{c}(-x^2ce^{-\frac{c^2x^2}{2a^2}}+x^2ce^{-\frac{c^2x^2}{2}})+4(a-1)\frac{1}{c^2}\right]\notag\\&=
    \frac{\sigma}{\sqrt{2\pi}}\left[-\frac{1}{c^2}\left(a^2e^{-\frac{c^2x^2}{2a^2}}-e^{-\frac{c^2x^2}{2}}\right)+(-x^2e^{-\frac{c^2x^2}{2a^2}}+x^2e^{-\frac{c^2x^2}{2}})+4(a-1)\frac{1}{c^2}\right]\notag\\&<
    \frac{\sigma}{\sqrt{2\pi}}\left[-\frac{1}{c^2}\left(a^2e^{-\frac{c^2x^2}{2a^2}}-e^{-\frac{c^2x^2}{2}}\right)+(-a^2x^2e^{-\frac{c^2x^2}{2a^2}}+x^2e^{-\frac{c^2x^2}{2}})+4(a-1)\frac{1}{c^2}\right]\notag\\&=
    \frac{\sigma}{\sqrt{2\pi}}\Big[\underbrace{(\frac{1}{c^2}+x^2)}_{>0}\underbrace{\left(e^{-\frac{c^2x^2}{2}}-a^2e^{-\frac{c^2x^2}{2a^2}}\right)}_{<0,\, \text{for}\, a \in (0,1)}+\underbrace{4(a-1)\frac{1}{c^2}}_{<0}\Big] \notag\\&< 0,
    \label{eq:pvpc}
\end{align}
and
\begin{align}
    \frac{\partial V}{\partial a} &= \sigma\left[\frac{1}{c}\frac{\partial U(cx,a)}{\partial a}-\frac{\partial U(x,a)}{\partial a} +\frac{4}{\sqrt{2\pi}}(1-\frac{1}{c})\right]\notag\\&= \frac{\sigma}{\sqrt{2\pi}}\left[\frac{1}{c}(2a+\frac{c^2x^2}{a})e^{-\frac{c^2x^2}{2a^2}}-(2a+\frac{x^2}{a})e^{-\frac{x^2}{2a^2}}\right]+\frac{4\sigma}{\sqrt{2\pi}}(1-\frac{1}{c})\notag
    \\&= \frac{\sigma}{\sqrt{2\pi}}\left[(\frac{2a}{c}+\frac{cx^2}{a})e^{-\frac{c^2x^2}{2a^2}}-(2a+\frac{x^2}{a})e^{-\frac{x^2}{2a^2}}\right]+\frac{4\sigma}{\sqrt{2\pi}}(1-\frac{1}{c}).
\label{eq:pvpa}    
\end{align}
Substitute $c=1$ into the Eq.~(\ref{eq:pvpa}),
\begin{equation}
   \left. \frac{\partial V}{\partial a} \right|_{c=1} = 0.
   \label{eq:pvpac1}
\end{equation}
For c>1, we have:
\begin{equation}
    \frac{\partial V}{\partial a\partial c} = \frac{\sigma}{\sqrt{2\pi}}e^{-\frac{c^2x^2}{2a^2}}\left[-\frac{2a}{c^2}-\frac{x^2}{a}+\frac{c^2x^4}{a^3}\right] +\frac{4\sigma}{\sqrt{2\pi}c^2} .
\end{equation}
When $-\frac{2a}{c^2}-\frac{x^2}{a}+\frac{c^2x^4}{a^3} \geq 0$, its obvious that $ \frac{\partial V}{\partial a\partial c} > 0$.
When $-\frac{2a}{c^2}-\frac{x^2}{a}+\frac{c^2x^4}{a^3} < 0$, we have:
\begin{align}
    \frac{\partial V}{\partial a\partial c} &= \frac{\sigma}{\sqrt{2\pi}}e^{-\frac{c^2x^2}{2a^2}}\left[-\frac{2a}{c^2}-\frac{x^2}{a}+\frac{c^2x^4}{a^3}\right] +\frac{4\sigma}{\sqrt{2\pi}c^2} \notag\\&\geq 
    \frac{\sigma}{\sqrt{2\pi}}\left[-\frac{2a}{c^2}-\frac{x^2}{a}+\frac{c^2x^4}{a^3}\right] +\frac{4\sigma}{\sqrt{2\pi}c^2}
    \notag\\&\geq \frac{\sigma}{\sqrt{2\pi}}\frac{-\frac{9}{4}a+4}{c^2} > 0.
\end{align}
Due to the fact that Eq.~(\ref{eq:pvpac1}): $\left. \frac{\partial V}{\partial a} \right|_{c=1} = 0$ and $ \frac{\partial V}{\partial a\partial c} > 0$, for any $c>1$, we have Eq.~(\ref{eq:pvpa}): 
\begin{equation}
     \frac{\partial V}{\partial a} > 0.
\end{equation}
Substitute $c=1$ and $a=1$ into the Eq.~(\ref{eq:v}) separately, we have:
\begin{equation}
    V(x,1,c) = 0,\label{eq:va1}
\end{equation}
and
\begin{equation}
    V(x,a,1) = 0.\label{eq:vc1}
\end{equation}
Therefore, from Eq.~(\ref{eq:pvpc}): $\frac{\partial V}{\partial c} < 0$, Eq.~(\ref{eq:pvpa}): $\frac{\partial V}{\partial c} > 0$, Eq.~(\ref{eq:va1}): $V(x,1,c) = 0$, and Eq.~(\ref{eq:vc1}):$V(x,a,1) = 0$, for $0<a<1$,$c>1$, we can conclude:
\begin{equation}
    V(x,a,c) < 0,\label{eq:v0}
\end{equation}
for any $x>0$.
Therefore, 
\begin{equation}
    (\MakeUppercase{\romannumeral 2}) = V(b,a,c) < 0,
    \label{eq:II}
\end{equation}
always stands up for any $0<a<1$, $b>0$, and $c>1$.

Combining Eq.~(\ref{eq:I}) and Eq.~(\ref{eq:II}), we conclude that Eq.~(\ref{eq:outin}):
\begin{equation}
    \mathbb{E}_{\text{out}}[\bar{z_i}-z_i]-\mathbb{E}_{\text{in}}[\bar{z_i}-z_i] < 0.
\end{equation}

\subsection{Related Concepts}
\label{A4.1.1}
\subsubsection{Loss Barrier}

A loss barrier \cite{frankle2020linear} between two models refers to a scenario where the optimization landscape, as defined by the loss function, presents a considerable and challenging obstacle for transitioning from one model to another. When attempting to move from one model to another, the goal is to adjust the parameters in a way that leads to improved performance on a specific task. However, if there exists a loss barrier between the two models, this means that making parameter updates to transition from the first model to the second model might involve encountering a region in the parameter space where the loss function increases significantly.

\subsubsection{Git Rebasin}

The postulate of Git Rebasin \cite{ainsworth2023git} methodology posits that a substantial subset of Stochastic Gradient Descent (SGD) solutions attained through the customary training regimen of neural networks belongs to a discernible collection, wherein the constituent elements can be systematically permuted. This permutation yields a configuration wherein no loss barrier in loss landscape exist along the trajectory of linear interpolation connecting any two permuted constituents.

\subsubsection{Ensemble Learning}

Ensemble methods have been a longstanding approach to enhancing model performance by combining the predictions of multiple models. Bagging~\cite{breiman1996bagging} and Boosting~\cite{freund1996experiments} are classic ensemble techniques that aim to mitigate overfitting and bias in predictions. Deep ensembles, as introduced by \cite{lakshminarayanan2017simple}, leverage multiple neural networks with different initializations to capture model uncertainty and encourage diverse parameter sampling. SSLC~\cite{vyas2018out} and kFolden ~\cite{li2021kfolden} combined the traditional idea based on the difference in data samples with the data leave-out approach to construct deep ensembles. All of these ensembles have been effective in the OOD detection problem.

\subsection{Sinkhorn Distance and Coupling Matrix}
\label{A3}

\subsubsection{Computation of Coupling Matrix}
The feature representations generated by the two models are considered as distributions $g_{H_1}(D)$ and $g_{H_2}'(D)$. The coupling matrix 
$\mathbf{P}_{H_1,H_2}$ represents how much probability mass from one point in support of $g_{H_1}(D)$ is assigned to a point in support of $g_{H_2}'(D)$. 
For a coupling matrix $\mathbf{P}_{H_1,H_2}$, all its columns must add to a vector containing the probability masses for $g_{H_1}(D)$, denoted as $\mathbf{v}_{H_1}$
, and all its rows must add to a vector with the probability masses for $g_{H_2}'(D)$,  denoted as $\mathbf{v}_{H_2}$.

The calculation of the total overhead of this assignment also relies on another cost matrix $\mathbf{C}$, which describes the cost of assigning a point in support of $g_{H_1}(D)$ to every single point in support of $g_{H_2}'(D)$. We usually use the $L^p$ distance~(p=2 in this work) between the feature representations of the samples to obtain the cost matrix.

The ultimate goal is to optimize:
\begin{equation*}
\begin{aligned}
        &\min_{\mathbf{P}_{H_1,H_2}}\langle \mathbf{C},\mathbf{P}_{H_1,H_2}\rangle\\
        \text{subject~to}~ & \mathbf{P}_{H_1,H_2}~\mathbf{1} = \mathbf{v}_{H_1},\\
        & \mathbf{1}^T~\mathbf{P}_{H_1,H_2} = \mathbf{v}_{H_2}^T
\end{aligned}
\end{equation*}
The minimum is known as Wasserstein distance. However, it is hard to compute because of computational complexity and non-convexity. Sinkhorn distance~\cite{cuturi2013sinkhorn} is an approximation to Wasserstein distance, which introduces an entropic regularization to make the problem convex, and therefore, can be solved iteratively.
Thus the problem is transformed into:
\begin{equation*}
\begin{aligned}
        &\min_{\mathbf{P}_{H_1,H_2}}\langle \mathbf{C},\mathbf{P}_{H_1,H_2}\rangle + \epsilon\sum_{ij}{\mathbf{P}_{H_1,H_2}}_{ij}\log{\mathbf{P}_{H_1,H_2}}_{ij}\\
        subject~to~ & \mathbf{P}_{H_1,H_2}~\mathbf{1} = \mathbf{v}_{H_1},\\
        & \mathbf{1}^T~\mathbf{P}_{H_1,H_2} = \mathbf{v}_{H_2}^T
\end{aligned}
\end{equation*}

Increasing $\epsilon$ will make the coupling matrix smoother.
The solution to this optimization problem can be written as $\mathbf{P}_{H_1,H_2} = diag(u)\mathbf{K}diag(v)$, where $\mathbf{K} = e^{-\lambda\mathbf{C}} $ is a kernel matrix.
$u$ and $v$ are updated with the iteration:
\begin{equation*}
\begin{aligned}
        & u^{(k+1)} = \frac{\mathbf{v}_{H_1}}{\mathbf{K}v^{(k)}}\\
        & v^{(k+1)} = \frac{\mathbf{v}_{H_2}}{\mathbf{K}^Tu^{(k+1)}}
\end{aligned}
\end{equation*}

After multiple iterations (100 in this work), the final coupling matrix $\mathbf{P}_{H_1,H_2}$ is obtained.
\subsubsection{Strength of Regularization in Sinkhorn Distance}
The strength of regularization ($\epsilon$) is set according to analysis in \cite{cuturi2013sinkhorn} that requires taking $\epsilon^{-1}$ in order of $\log n/p$, where $n$ is the number of samples and $p$ is the tolerance of approximation. In this paper, $n = 512$ and $p = 0.0001$, thus, we should have $\epsilon^{-1} > 15.44$. In experiments of this paper, we set $\epsilon$ to 0.05 to satisfy the above constraint. Under this constraint, the approximation precision of the Sinkhorn distance is sufficient to support our observation of self-coupling. 

\subsection{Self-Coupling Index Table}
\label{A4}
\begin{table*}[h]
  \caption{Self-Coupling Index for models trained under different initialization and training strategies. The model structure is ResNet-18. The dataset is CIFAR10.}
  \label{sample-table}
  \centering
  \begin{tabular}{lcccccc}
    \toprule
   &SupCE   & SupCon             & SimCLR  & MoCo & RotNet & JigClu\\
    \midrule
    SupCE & \textbf{0.861}  & 0.203   &0.091  & 0.094 & 0.107 & 0.163 \\
    SupCon~\cite{khosla2020supervised}  & 0.214 & \textbf{0.877}  & 0.107 & 0.193 & 0.207 & 0.187   \\
    SimCLR~\cite{chen2020simple} & 0.094   & 0.089  & \textbf{0.834} & 0.367 & 0.147 & 0.139 \\
   MoCo~\cite{he2020momentum} & 0.097 & 0.209  & 0.339 & \textbf{0.913} & 0.329 & 0.096  \\
   RotNet~\cite{gidaris2018unsupervised} & 0.119 & 0.207  & 0.165 & 0.311 & \textbf{0.987} & 0.165  \\
   JigClu~\cite{chen2021jigsaw} & 0.170 & 0.175  & 0.126 & 0.101 & 0.160 & \textbf{0.915}  \\
    \bottomrule
  \end{tabular}
  \label{tab:sci-res18-cifar}
\end{table*}
\begin{table*}[h]
  \caption{Self-Coupling Index for models trained under different initialization and training strategies. The model structure is ResNet-50. The dataset is CIFAR10.}
  \label{sample-table}
  \centering
  \begin{tabular}{lcccccc}
    \toprule
   &SupCE   & SupCon             & SimCLR  & MoCo & RotNet & JigClu\\
    \midrule
    SupCE & \textbf{0.841}  & 0.197   &0.106  & 0.076 & 0.097 & 0.168 \\
    SupCon~\cite{khosla2020supervised}  & 0.200 & \textbf{0.854}  & 0.098 & 0.199 & 0.163 & 0.157   \\
    SimCLR~\cite{chen2020simple} & 0.099   & 0.069  & \textbf{0.812} & 0.316 & 0.112 & 0.119 \\
   MoCo~\cite{he2020momentum} & 0.079 & 0.185  & 0.319 & \textbf{0.891} & 0.289 & 0.086  \\
   RotNet~\cite{gidaris2018unsupervised} & 0.113 & 0.187  & 0.155 & 0.293 & \textbf{0.965} & 0.143  \\
   JigClu~\cite{chen2021jigsaw} & 0.164 & 0.149  & 0.117 & 0.081 & 0.153 & \textbf{0.865}  \\
    \bottomrule
  \end{tabular}
  \label{tab:sci-res50-cifar}
\end{table*}
\begin{table*}[h]
  \caption{Self-Coupling Index for models trained under different initialization and training strategies. The model structure is ResNet-50. The training dataset is ImageNet-1K.}
  \label{sample-table}
  \centering
  \begin{tabular}{lcccccc}
    \toprule
   &SupCE   & SupCon             & SimCLR  & MoCo & RotNet & JigClu\\
    \midrule
    SupCE & \textbf{0.721}  & 0.021   &0.047  & 0.036 & 0.031 & 0.078 \\
   SupCon~\cite{khosla2020supervised}  & 0.069 & \textbf{0.734}  & 0.036 & 0.089 & 0.063 & 0.059   \\
    SimCLR~\cite{chen2020simple} & 0.039   & 0.028  & \textbf{0.757} & 0.196 & 0.062 & 0.043 \\
   MoCo~\cite{he2020momentum} & 0.028 & 0.067  & 0.183 & \textbf{0.699} & 0.186 & 0.016  \\
   RotNet~\cite{gidaris2018unsupervised} & 0.043 & 0.063  & 0.053 & 0.164 & \textbf{0.765} & 0.127  \\
   JigClu~\cite{chen2021jigsaw} & 0.049 & 0.057  & 0.034 & 0.011 & 0.053 & \textbf{0.711}  \\
    \bottomrule
  \end{tabular}
  \label{tab:sci-res50-im}
\end{table*}
\begin{table*}[h]
  \caption{Self-Coupling Index for models trained under different initialization and training strategies. The model structure is ViT-B. The dataset is CIFAR10.}
  \label{sample-table}
  \centering
  \begin{tabular}{lcccc}
    \toprule
   &SupCE   & MoCo v3 & MAE & DINO        \\
    \midrule
    SupCE & \textbf{0.769}  & 0.068   &0.046  & 0.036 \\
    MoCo v3~\cite{chen2021empirical}  & 0.111 & \textbf{0.862}  & 0.031 & 0.159    \\
    MAE~\cite{he2022masked} & 0.049   & 0.016  & \textbf{0.887} & 0.036  \\
   DINO~\cite{caron2021emerging} & 0.031 & 0.185  & 0.035 & \textbf{0.791} \\
   
    \bottomrule
  \end{tabular}
  \label{tab:sci-vit}
\end{table*}
In Table~\ref{tab:sci-res18-cifar} and~\ref{tab:sci-res50-cifar}, we report the Self-Coupling Index between some models with representative training criterion trained on CIFAR10 dataset with ResNet-18 and ResNet-50~\cite{he2016deep}, respectively. There is a large Self-Coupling Index between the same training method and a smaller Self-Coupling Index between models with different training methods.

In Table~\ref{tab:sci-res50-im} and~\ref{tab:sci-vit}, we report the Self-Coupling Index between some models with representative training criterion trained on ImageNet dataset with ResNet-50~\cite{he2016deep} and ViT-B~\cite{dosovitskiy2020image}, respectively. Due to the large amount of Imagenet data, we take a balanced part of the dataset to calculate the self-coupling index.

\begin{table*}[h]
  \caption{ Comparison with state-of-the-art OOD detection methods. All results are in percentages. *: Outlier Exposure based model.}
  \resizebox{\textwidth}{!}{
  \centering    
  \tabcolsep 1pt
  \begin{tabular}{lcccccccccccc}
    \toprule

     \multicolumn{13}{c}{\textbf{OOD Dataset}}\\
   & \multicolumn{2}{c}{\textbf{SVHN}}   & \multicolumn{2}{c}{\textbf{LSUN}}              & \multicolumn{2}{c}{\textbf{iSUN}}      & \multicolumn{2}{c}{\textbf{Texture}}     & \multicolumn{2}{c}{\textbf{Places365}}     & \multicolumn{2}{c}{\textbf{Average}}  \\
    \cmidrule(r){2-3}  \cmidrule(r){4-5}  \cmidrule(r){6-7}  \cmidrule(r){8-9}  \cmidrule(r){10-11}  \cmidrule(r){12-13}
      Methods   & FPR95     & AUROC & FPR95     & AUROC & FPR95     & AUROC       & FPR95     & AUROC & FPR95     & AUROC & FPR95     & AUROC \\
    \midrule
    FeatureNorm~\cite{yu2023block} & 7.13  & 98.65 & 27.08 & 95.25 & 26.02 & 95.38 & 31.18 & 92.31 & 62.54 & 84.62 & 30.79 & 93.24 \\
DOE*~\cite{wang2023outofdistribution}        & 2.65  & 99.36 & 0     & 99.89 & 0.75  & 99.67 & 7.25  & 98.47 & 15.1  & 96.53 & 5.15  & 98.78 \\
CIDER~\cite{ming2023how}       & 3.04  & 99.5  & 4.1   & 99.14 & 15.94 & 97.1  & 13.19 & 97.3  & 26.6  & 94.64 & 12.57 & 97.55 \\
SHE~\cite{zhang2023outofdistribution}         & 5.87  & 98.74 & 6.67  & 98.42 & 4.16  & 98.85 &       &       & 6.31  & 98.7  &       &       \\
DICE~\cite{sun2022dice}       & 25.99 & 95.9  & 3.91  & 99.2  & 4.36  & 99.14 & 41.9  & 88.18 & 48.59 & 89.11 & 24.95 & 94.3  \\
ASH-S~\cite{djurisic2023extremely}       & 6.51  & 98.56 & 4.96  & 98.92 & 5.17  & 98.9  & 24.34 & 95.09 & 48.45 & 88.31 & 17.89 & 95.96 \\
\textbf{MC Ens.}       & 1.35  & 99.7  & 1.45  & 99.8  & 7.88  & 98.09 & 4.07  & 99.05 & 13.19 & 97.01 & 5.58  & 98.73\\
    \bottomrule
  \end{tabular}}
  \label{tab:sota}

\end{table*}

\subsection{Scoring Methods}
\label{A4.1}
(1) MSP~\cite{hendrycks2017baseline}: using maximum softmax probability as detection scoring metric, and ID data point will have higher softmax probability. 

\begin{equation}
s_{\text{MSP}}(\mathbf{x})=\max_k \operatorname{Softmax}(W g_H(\mathbf{x})+\mathbf{b})_k,
\end{equation}
where $W$ and $b$ is parameter for output layer.

(2) Mahalanobis distance~\cite{lee2018simple}: Mahalanobis distance takes into account the covariance of the class distribution. The data point has a high Mahalanobis distance from the distribution is considered OOD. 

\begin{equation}
s_{\text {Mahal.}}(\mathbf{x}):=\max _k-\left(g_H(\mathbf{x})-\hat{\mu}_k\right)^{\top} \hat{\Sigma}\left(g_H(\mathbf{x})-\hat{\mu}_k\right)
\end{equation}
where $\hat{\mu}_k$ and $\hat{\Sigma}$, are the estimated feature vector mean and covariance for classes.

(3)Energy~\cite{liu2020energy}: Energy score uses the energy-based model to score the feature representation. 

\begin{equation}
s_{\text {Energy }}(\mathbf{x})=-\log \sum_{k=1}^K \exp \left(\mathbf{w}_i^{\top} g_H(\mathbf{x})+b_i\right)
\end{equation}

(4)KNN~\cite{sun2022out}: it is a non-parameter approach that computes the k-nearest neighbor distance between test input embedding and training set embeddings, using a threshold to determine OOD.
\begin{equation}
s_{\text{KNN}}\left(g_H(\mathbf{x})^* ; k\right)=\mathbf{1}\left\{-r_k\left(g_H(\mathbf{x})^*\right) \geq \lambda\right\},
\end{equation}
where where $r_k\left(\mathbf{z}^*\right)=\left\|\mathbf{z}^*-\mathbf{z}_{(k)}\right\|_2$ is the distance to the k-th nearest neighbor

\subsection{CIFAR10 Benchmark on ResNet-50}
\label{A5}

\begin{table*}[h]
  \caption{\textbf{Results on CIFAR10 Benchmark with ResNet-50.} Comparison with competitive OOD detection methods. All results are in percentages.}
  \resizebox{\textwidth}{!}{
  \centering    
  \tabcolsep 1pt
  \begin{tabular}{lcccccccccccc}
    \toprule

     \multicolumn{13}{c}{\textbf{OOD Dataset}}\\
   & \multicolumn{2}{c}{\textbf{SVHN}}   & \multicolumn{2}{c}{\textbf{LSUN}}              & \multicolumn{2}{c}{\textbf{iSUN}}      & \multicolumn{2}{c}{\textbf{Texture}}     & \multicolumn{2}{c}{\textbf{Places365}}     & \multicolumn{2}{c}{\textbf{Average}}  \\
    \cmidrule(r){2-3}  \cmidrule(r){4-5}  \cmidrule(r){6-7}  \cmidrule(r){8-9}  \cmidrule(r){10-11}  \cmidrule(r){12-13}
      Methods   & FPR95     & AUROC & FPR95     & AUROC & FPR95     & AUROC       & FPR95     & AUROC & FPR95     & AUROC & FPR95     & AUROC \\
    \midrule
    ODIN~\cite{liang2018enhancing}  & 18.34 &95.68 &7.10 &98.67 &28.17 &94.69 &53.26 &87.42& 59.07 &88.57 &33.19 &93.01\\ 
    SSD+~\cite{sehwag2021ssd} &1.07 &99.80 &5.25 &98.87& 29.75 &95.64 &9.99 &97.97 &25.37 &94.85 &14.29 &97.42 \\
    CSI~\cite{tack2020csi} &39.12 &93.96 &4.88 &99.00 &10.41 &98.02 &29.31 &94.41 &36.23 &93.13 &23.99 &95.70\\
    MSP~\cite{hendrycks2017baseline} &53.36 &92.31 &48.46 &93.60 &53.86 &92.07 &60.34 &89.01 &57.32 &89.16 &54.67 &91.23\\
    Mahalanobis~\cite{lee2018simple} &9.56 &97.36 &59.86 &78.37 & 16.02 &96.41 &19.31 &94.30 &69.67 &73.56 &34.88 &88.00\\
    Energy~\cite{liu2020energy} &48.06 &92.60 &11.85 &97.62 &26.52 &95.55 &43.32 &93.33 & 41.37 &91.35 &34.22 &94.09\\
    KNN~\cite{sun2022out} &23.19 &95.89 &23.29 &96.18 &21.55 &96.11 &23.90 &95.12 &43.97 &91.23 &27.18 & 94.90\\
    KNN+\cite{sun2022out} &2.65 &99.43 &1.98 &99.38 &19.36 &96.71 &7.11 &98.75 &19.12 &96.31 &10.04 &98.11\\
    \midrule
    \textbf{MC Ens.+MSP} & 42.37 & 91.78 & 43.45 & 92.11 & 43.36 & 92.32 & 43.86 &92.44 & 49.13 & 90.00 & 44.43 & 91.73\\
    \textbf{MC Ens.+Mahala.} &3.64 & 98.99 & 41.32 & 94.35 & 18.55 & 94.97 & 12.27& 94.81 & 24.68 & 91.02 & 20.09 & 94.82\\
    \textbf{MC Ens.+Energy} &34.94 & 92.59 & 6.01& 99.06& 17.99 &96.62 &23.98 &91.91 & 31.02 & 92.98 & 22.79 & 94.63\\
    \textbf{MC Ens.+KNN} & \textbf{0.89} & \textbf{99.81} & \textbf{0.24} & \textbf{99.91} & \textbf{6.96} & \textbf{98.23} & \textbf{5.13} & \textbf{98.86} & \textbf{12.39} & \textbf{97.75} & \textbf{5.12} & \textbf{98.91}\\
    \bottomrule
  \end{tabular}}
  \label{tab:cifar-res50}

\end{table*}

As shown in Table~\ref{tab:cifar-res50}, we trained 3 different ResNet-50s for the MC Ensemble, the training configuration is the same as ResNet-18 except the batch size is set to 256. The result is consistent with Table~\ref{tab:cifar}. We notice that the MC Ensemble+MSP result is lower than the one in ResNet-18, we argue that this is because larger models tend to give a more over-confident prediction.
\subsection{Comparison with Naive Ensemble on Mahalanobis Distance and Energy}
\label{A6}

\begin{table*}[h]
  \caption{\textbf{Comparison with naive ensemble.} Models in naive ensemble are trained from different weight initialization. All results are in percentages.}

  \centering    
  \tabcolsep 1pt
  \begin{tabular}{lcccccccccccc}
    \toprule

     \multicolumn{13}{c}{\textbf{OOD Dataset}}\\
   & \multicolumn{2}{c}{\textbf{SVHN}}   & \multicolumn{2}{c}{\textbf{LSUN}}              & \multicolumn{2}{c}{\textbf{iSUN}}      & \multicolumn{2}{c}{\textbf{Texture}}     & \multicolumn{2}{c}{\textbf{Places365}}     & \multicolumn{2}{c}{\textbf{Average}}  \\
    \cmidrule(r){2-3}  \cmidrule(r){4-5}  \cmidrule(r){6-7}  \cmidrule(r){8-9}  \cmidrule(r){10-11}  \cmidrule(r){12-13}
      Methods   & FPR95     & AUROC & FPR95     & AUROC & FPR95     & AUROC       & FPR95     & AUROC & FPR95     & AUROC & FPR95     & AUROC \\
    \midrule
    \multicolumn{13}{c}{\textbf{Mahalanobis distance}}\\
    3$\times$SupCE  & 8.71 &97.96 &59.29 &87.96 & 31.55 &93.26 &21.57 &93.72 &73.90 &71.14 & 39.00 & 88.80\\ 
    \textbf{MC Ens.} &2.09 & 99.48 & 43.35 & 93.79 & 21.59&94.77&14.31& 94.68 & 27.68 & 89.88 & 21.80 & 94.52\\
    \multicolumn{13}{c}{\textbf{Energy score}}\\
    3$\times$SupCE &51.29 &91.83 &11.15 &97.79  &25.88 &95.21 &53.55 &89.91 & 40.93 &92.01 &36.56 & 93.35\\
    \textbf{MC Ens.} &34.99 & 92.58 & 6.05& 99.05& 17.96 &96.59 &23.97 &91.92 & 33.02 & 92.37 & 23.20 & 94.50\\
    \bottomrule
  \end{tabular}
  \label{tab:cifar2MDEN}

\end{table*}

The comparison of the naive ensemble with 3 cross-entropy trained ResNet-18 and MC Ensemble on Mahalanobis distance~\cite{lee2018simple} and Energy socre~\cite{liu2020energy} is shown in Table~\ref{tab:cifar2MDEN}. MC Ensemble consistently outperforms naive ensemble on these scoring metrics.

\subsection{ImageNet Benchmark on ViT-B}
\label{A7}
\begin{table*}[h]
  \caption{\textbf{Results on ImageNet Benchmark with ViT-B~\cite{dosovitskiy2020image}}.  All results are in percentages. Scoring metric is KNN.}

  \centering    
  \tabcolsep 1pt
  \begin{tabular}{lcccccccccc}
    \toprule

     \multicolumn{11}{c}{\textbf{OOD Dataset}}\\
   & \multicolumn{2}{c}{\textbf{iNaturalist}}   & \multicolumn{2}{c}{\textbf{SUN}}              & \multicolumn{2}{c}{\textbf{Places}}      & \multicolumn{2}{c}{\textbf{Textures}}   & \multicolumn{2}{c}{\textbf{Average}}  \\
    \cmidrule(r){2-3}  \cmidrule(r){4-5}  \cmidrule(r){6-7}  \cmidrule(r){8-9}  \cmidrule(r){10-11}  
      Methods   & FPR95     & AUROC & FPR95     & AUROC & FPR95     & AUROC       & FPR95     & AUROC & FPR95     & AUROC  \\
    \midrule
    
    ViT-B(SupCE) & 8.41 & 97.23 & 49.98 & 86.32 & 37.98 &91.37 & 56.24 &85.71 & 38.15 & 90.16\\
    \midrule

    \textbf{MC ViT Ens.} & \textbf{7.99} & \textbf{97.73} & \textbf{43.72} & \textbf{90.69} &\textbf{35.89} &\textbf{91.02} & \textbf{34.65} & \textbf{91.77}  & \textbf{30.56} & \textbf{92.80}\\
    \bottomrule
  \end{tabular}

  \label{tab:imagenet-vit}

\end{table*}

We fine-tune 3 different ViT-B models which are trained with cross-entropy, MoCo v3~\cite{chen2021empirical}, and Masked-Autoencoder~\cite{he2022masked} to build a MC ViT Ensemble. The weights are imported from their original repositories.
As shown in Table~\ref{tab:imagenet-vit}, MC ViT Ensemble still consistently outperforms vanilla ViT.

\subsection{ImageNet Benchmark on ResNet-50}
\label{A8}
As shown in Table~\ref{tab:imagenet-res50}, we report the ImageNet Benchmark results on ResNet-50. We notice Mahalanobis distance scoring metric almost crashes on ImageNet benchmark, this may be because the Mahalanobis distance leverages the class center information, and in Imagenet Benchmark, there are 1000 class centers, which is hard to determine which class a sample belongs to regardless of its distribution. MC Ensemble is not able to improve this.
\begin{table*}[h]
  \caption{\textbf{Results on ImageNet Benchmark.}  All results are in percentages. Some of the baseline results are from~\cite{sun2022out}.}

  \centering    
  \tabcolsep 1pt
  \begin{tabular}{lcccccccccc}
    \toprule

     \multicolumn{11}{c}{\textbf{OOD Dataset}}\\
   & \multicolumn{2}{c}{\textbf{iNaturalist}}   & \multicolumn{2}{c}{\textbf{SUN}}              & \multicolumn{2}{c}{\textbf{Places}}      & \multicolumn{2}{c}{\textbf{Textures}}   & \multicolumn{2}{c}{\textbf{Average}}  \\
    \cmidrule(r){2-3}  \cmidrule(r){4-5}  \cmidrule(r){6-7}  \cmidrule(r){8-9}  \cmidrule(r){10-11}  
      Methods   & FPR95     & AUROC & FPR95     & AUROC & FPR95     & AUROC       & FPR95     & AUROC & FPR95     & AUROC  \\
    \midrule
    ODIN & 47.66 & 89.66 & 60.15 & 84.59 & 50.23 & 85.62 &  67.89 &81.78 & 56.48 & 85.41\\
    SSD+ & 57.16 & 87.77 & 78.23 & 73.10 & 36.37 & 88.52 & 81.19 &70.97 &63.24 &80.09\\
    MSP  & 54.99&  87.74 & 70.83 & 80.86 & 68.00 & 79.61 & 73.99 & 79.76  & 66.95 & 81.99\\
    Mahalanobis & 97.00 & 52.65 & 98.50 &42.41 & 55.80 &85.01 & 98.40 & 41.79 & 87.43 &55.47\\
    Energy & 55.72 & 89.95 & 59.26 & 85.89 & 53.72 & 85.99 & 64.92 & 82.86 & 58.41 & 86.17\\
    KNN & 59.00 & 86.47 & 68.82 & 80.72 & 11.77 &97.07 & 76.28 &75.76 & 53.97 &85.01\\
    KNN+ & 30.18 & 94.89 & 48.99 & 88.63 &15.55 & 95.40 & 59.15 & 84.71 & 38.47 & 90.91\\
    \midrule
    \textbf{MC Ens.+Mahala.} & 98.00 & 52.15 & 100.00 & 50.95 & 97.65 & 51.24 & 100.00 & 48.97 & 98.91 & 50.83\\
     \textbf{MC Ens.+Energy} & 38.45 & 92.75 & 43.98 & 90.33 & 37.69 & 91.88 & 48.96 & 87.91 & 42.27 & 90.72\\
    \textbf{MC Ens.+KNN} & \textbf{15.39} & \textbf{96.78} & \textbf{42.97} & \textbf{90.35} &\textbf{54.89} &\textbf{87.34} & \textbf{9.54} & \textbf{97.77}  & \textbf{30.69} & \textbf{93.06}\\
    \bottomrule
  \end{tabular}

  \label{tab:imagenet-res50}

\end{table*}
%

\subsection{Limitations}
\label{A10}
This paper proposes to aggregate multiple models trained with different tasks to form a multi-comprehension ensemble for better OOD detection performance. The limitation of this paper can be that: (1) In current work, the task/criteria pool we have explored cannot be described as large, and this makes it possible for us to miss the opportunity to find a more powerful MC ensemble. As more and more training task/criteria being proposed, the task/criteria pool needs further study. (2) The computation overhead of MC Ensemble is still $M \times$ compared to a single standalone model with the same backbone. In the case of constrained computation resources, the latency may increase.



\end{document}